\documentclass{article}

\usepackage{microtype}
\usepackage{graphicx}
\usepackage{subfig}
\usepackage{booktabs} 

\usepackage{hyperref}

\usepackage[utf8]{inputenc} 
\usepackage[T1]{fontenc}    
\usepackage{hyperref}       
\usepackage{booktabs}       
\usepackage{amsfonts}       
\usepackage{nicefrac}       
\usepackage{microtype}      
\usepackage[hang,flushmargin,symbol]{footmisc}
\usepackage{algorithmic}
\renewcommand{\algorithmicrequire}{\textbf{Input:}}
\renewcommand{\algorithmicensure}{\textbf{Output:}}
\usepackage{subfig}

\usepackage{xspace}
\usepackage{natbib}
\usepackage{graphicx}
\usepackage{mathtools}
\usepackage{marvosym,listings,etoolbox}

\newcommand{\alg}{{\textsc{CRAIG}}\xspace}

\newcommand{\strongf}{\ensuremath{\mu}}
\newcommand{\opt}{\ensuremath{w_*}\xspace}

\newcommand{\xhdr}[1]{{\noindent\bfseries #1}.}

\usepackage{amsthm,amssymb,amsmath,mathtools}

\newtheorem{theorem}{Theorem}

\newtheorem{lemma}[theorem]{Lemma}

\newcommand{\hide}[1]{}

\theoremstyle{definition}

\usepackage[accepted]{icml2020}

\icmltitlerunning{Coresets for Data-efficient Training of Machine Learning Models {\textcolor{red}{\bf PAGE \thepage}}}

\begin{document}

\twocolumn[
\icmltitle{Coresets for Data-efficient Training of Machine Learning Models}



\icmlsetsymbol{equal}{*}

\begin{icmlauthorlist}
\icmlauthor{Baharan Mirzasoleiman}{uc}
\icmlauthor{Jeff Bilmes}{uw}
\icmlauthor{Jure Leskovec}{st}
\end{icmlauthorlist}

\icmlaffiliation{uc}{Department of Computer Science, University of California, Los Angeles, USA}
\icmlaffiliation{st}{Department of Computer Science, Stanford University, Stanford, USA}
\icmlaffiliation{uw}{Department of Electrical Engineering, University of Washington, Seattle, USA}

\icmlcorrespondingauthor{Baharan Mirzasoleiman}{baharan@cs.ucla.edu}

\icmlkeywords{Machine Learning, ICML}

\vskip 0.3in
]



\printAffiliationsAndNotice{}  

\begin{abstract} 
Incremental gradient (IG) methods, such as stochastic gradient descent and its variants 
are commonly used for large scale optimization in machine learning. Despite the sustained effort to make IG methods more data-efficient, it remains an open question how to select a training data subset that can theoretically and practically perform on par with the full dataset. Here we develop \alg, a method to select a weighted subset (or coreset) of training data that closely estimates the full gradient by maximizing a submodular function.
We prove that applying IG to this subset is guaranteed to converge to the (near)optimal solution with the same convergence rate as that of IG for convex optimization. As a result, \alg achieves a speedup that is inversely proportional to the size of the subset.
To our knowledge, this is the \textit{first} rigorous method for data-efficient training of \textit{general} machine learning models. 
Our extensive set of experiments show that \alg, while achieving practically the same solution, speeds up various IG methods by up to 6x for logistic regression and 3x for training deep neural networks\footnote[2]{Code available at {\small \url{https://github.com/baharanm/craig}}}.
\end{abstract}


\section{Introduction}
Mathematical optimization lies at the core of training large-scale machine learning systems, and is now widely used over massive data sets with great practical success, assuming sufficient data resources are available. Achieving this success, however, also requires large amounts of (often GPU) computing, as well as concomitant financial expenditures and energy usage \citep{strubell2019energy}. Significantly decreasing these costs without decreasing the learnt system's resulting accuracy is one of the grand challenges of machine learning and artificial intelligence today \citep{asi2019importance}.

Training machine learning models often reduces to 
optimizing a regularized empirical risk function. Given a convex loss $l$, and a $\mu$-strongly convex regularizer $r$, one aims to find model parameter vector $w_*$ over the parameter space $\mathcal{W}$ that minimizes the loss $f$ over the training data $V$:
\vspace{-0mm}
\begin{align}\label{eq:problem}
\opt \in {\arg\min}_{w \in \mathcal{W}}& f(w), \quad f(w) := \sum_{i\in V} f_i(w) + r(w),\nonumber\\ 
&f_i(w) = l(w, (x_i, y_i)), 
\vspace{-0mm}
\end{align}
where $V = \{1,\dots,n\}$ is an index set of the training data, and functions $f_i: \mathbb{R}^d \rightarrow \mathbb{R}$ are associated with training examples $(x_i, y_i)$, where $x_i\in \mathbb{R}^d$ is the feature vector, and $y_i$ is the point $i$'s label.

Standard Gradient Descent can find the minimizer of this problem, but requires repeated computations of the full gradient $\nabla f(w)$---sum of the gradients over all training data points/functions $i$---and is therefore prohibitive for massive data sets. This issue is further exacerbated in case of deep neural networks where gradient computations (backpropagation) are expensive. 
Incremental Gradient (IG) methods, such as Stochastic Gradient Descent (SGD) and its accelerated variants, including SGD with momentum \citep{qian1999momentum}, Adagrad \citep{duchi2011adaptive}, Adam~\citep{kingma2014adam}, SAGA~\citep{defazio2014saga}, and SVRG~\citep{johnson2013accelerating} iteratively estimate the gradient on random subsets/batches of training data. While this provides an unbiased estimate of the full gradient, the randomized batches introduce variance in the gradient estimate \citep{hofmann2015variance}, and therefore stochastic gradient methods are in general slow to converge \citep{johnson2013accelerating, defazio2014saga}. The majority of the work speeding up IG methods has thus primarily focused on reducing the variance of the gradient estimate (SAGA~\citep{defazio2014saga}, SVRG~\citep{johnson2013accelerating}, Katysha \citep{allen2017katyusha}) or more carefully selecting the gradient stepsize (Adagrad \citep{duchi2011adaptive}, Adadelta \citep{zeiler2012adadelta}, Adam~\citep{kingma2014adam}).

However, the direction that remains largely unexplored is how to carefully select a small subset $S \subseteq V$ of the full training data $V$, so that the model is trained only on the subset $S$ while still (approximately) converging to the globally optimal solution (i.e., the model parameters that would be obtained if training/optimizing on the full $V$). If such a subset $S$ can be quickly found, then this would directly lead to a speedup of $|V|/|S|$ (which can be very large if $|S| \ll |V|$) per epoch of IG.

There are four main challenges in finding such a subset $S$. First is that a guiding principle for selecting $S$ is unclear. For example, selecting training points close to the decision boundary might allow the model to fine tune the decision boundary, while picking the most diverse set of data points would allow the model to get a better sense of the training data distribution. Second is that finding $S$ must be fast, as otherwise identifying the set $S$ may take longer than the actual optimization, and so no overall speed-up would be achieved. Third is that finding a subset $S$ is not enough. One also has to decide on 
a gradient stepsize for each data point in $S$, as they affect the convergence. 
And last, while the method might work well empirically on some data sets, one also requires theoretical understanding and mathematical convergence guarantees.

Here we develop {\em Coresets for Accelerating Incremental Gradient descent (\alg)}, for selecting a subset of training data points to speed up training of large machine learning models. Our key idea is to select a weighted subset $S$ of training data $V$ that best approximates the full gradient of $V$.
We prove that the subset $S$ that minimizes an upper-bound on the error of estimating the full gradient maximizes a submodular facility location function. 
Hence, $S$ can be efficiently found using a fast greedy algorithm. 

We also provide theoretical analysis of \alg and prove its convergence. 
Most importantly, we show that any incremental gradient method (IG) on $S$ converges in the same number epochs as the same IG would on the full $V$, which means that we obtain a speed-up inversely proportional to the size of $S$. In particular, for a $\strongf$-strongly convex risk function and a subset $S$ selected by \alg that estimates the full gradient by an error of at most $\epsilon$, we prove that IG on $S$ with diminishing stepsize $\alpha_k=\alpha/k^\tau$ at epoch $k$ (with $0<\tau<1$ and $0<\alpha$), converges to an $2R\epsilon/\strongf^2$ neighborhood of the optimal solution at rate $\mathcal{O}(1/\sqrt{k})$. Here, $R=\min\{d_0, (r\gamma_{\max}C + \epsilon)/\strongf \}$ where $d_0$ is the initial distance to the optimum, $C$ is an upper-bound on the norm of the gradients, $r=|S|$, and $\gamma_{\max}$ is the largest weight for the elements in the subset obtained by \alg. Moreover, we prove that if in addition to the strong convexity, component functions have smooth gradients, IG with the same diminishing step size on subset $S$ converges to a $2\epsilon/\strongf$ neighborhood of the optimum solution at rate $\mathcal{O}(1/k^\tau)$.
 
The above implies that IG on $S$ converges to the same solution and in the same number of epochs as IG on the full $V$.
But because every epoch only uses a subset $S$ of the data, it requires fewer gradient computations and thus leads to a $|V|/|S|$ speedup over traditional IG methods, while still (approximately) converging to the optimal solution. 
We also note that \alg is complementary to various incremental gradient (IG) methods (SGD, SAGA, SVRG, Adam), and such methods can be used on the subset $S$ found by \alg.



We also demonstrate the effectiveness of \alg via an extensive set of experiments using logistic regression (a convex optimization problem) as well as training deep neural networks (non-convex optimization problems). We show that \alg speeds up incremental gradient methods, including SGD, SAGA, and SVRG. 
In particular, \alg while achieving practically the same loss and accuracy as the underlying incremental gradient descent methods, speeds up gradient methods by up to 6x for convex and 3x for non-convex loss functions. 

\section{Related Work} 
Convergence of IG methods has been long studied under various conditions \citep{zhi1994analysis, mangasariany1994serial, bertsekas1996incremental, solodov1998incremental, tseng1998incremental}, however IG's convergence rate has been characterized only more recently (see \citep{bertsekas2015incremental} for a survey). 
In particular, \cite{nedic2001convergence} provides a $\mathcal{O}(1/\sqrt{k})$ convergence rate for diminishing stepsizes $\alpha_k$ per epoch $k$ under a strong convexity assumption, and \cite{gurbuzbalaban2015random} proves a $\mathcal{O}(1/k^\tau)$ convergence rate with diminishing stepsizes $\alpha_k=\Theta(1/k^\tau)$ for $\tau\in (0,1]$ under an additional smoothness assumption for the components. 
While these works provide convergence on the full dataset, 
our analysis provides the same convergence rates on subsets obtained by \alg.
 
\hide{
It has been empirically observed that ordering of data significantly affects the convergence rate of IG.
However, finding a favorable ordering for IG has been a long standing open question. 
Among the few results are that of \citep{recht2012beneath} showing that 
without-replacement random sampling improves convergence of IG for least means squares problem, and the very recent result of \citep{gurbuzbalaban2017convergence} showing that
a Random Reshuffling (RR) method with iterate averaging and a diminishing stepsize $\Theta(1/k^\tau)$ for $\tau \!\in \!(1/2,\!1)$ converges at rate $\Theta(1/k^{2\tau})$ with probability one in the suboptimality of the objective value, thus improving upon the  $\Omega(1/k)$ rate of SGD. 
Contrary to the above randomized analysis, we propose the first deterministic ordering on the data points and empirically show that the ordering provided by \alg provides a significant speedup for the convergence of IG. 
}

Techniques for speeding up SGD, are mostly focused on variance reduction techniques 
\citep{roux2012stochastic, shalev2013stochastic, johnson2013accelerating,hofmann2015variance,allen2016exploiting},
and accelerated gradient methods when the regularization parameter is small \citep{frostig2015regularizing, lin2015universal,xiao2014proximal}. 
Very recently, \cite{hofmann2015variance,allen2016exploiting} exploited neighborhood structure to further reduce the variance of stochastic gradient descent and improve its running time. 
Our \alg method and analysis are complementary to variance reduction and accelerated methods. \alg can be applied to all these methods as well to speed them up.

Coresets are weighted subsets of the data, which guarantee that models fitting the coreset also provide a good fit for the original data. 
Coreset construction methods traditionally perform importance sampling with respect to sensitivity score, 
to provide high-probability solutions \cite{har2004coresets, lucic2017training, cohen2017input} for a particular problem, 
such as $k$-means and $k$-median clustering \cite{har2004coresets}, na\"{\i}ve Bayes and nearest-neighbors \cite{wei2015submodularity}, mixture models \cite{lucic2017training}, 
low rank approximation \cite{cohen2017input}, spectral approximation \cite{agarwal2004approximating, li2013iterative}, Nystrom methods \cite{agarwal2004approximating, musco2017recursive},
and Bayesian inference \cite{Campbell18_ICML}.
Unlike existing coreset construction algorithms, our method is not problem specific and can be applied for training general machine learning models.

\section{Coresets for Accelerating Incremental Gradient Descent (\alg)}
We proceed as follows: First, we define an objective function $L$ for selecting an optimal set $S$ of size $r$ that best approximates the gradient of the full training dataset $V$ of size $n$. Then, we show that $L$ can be turned into a submodular function $F$ and thus $S$ can be efficiently found using a fast greedy 
algorithm. Crucially, we also show that for convex loss functions the approximation error between the estimated and the true gradient can be efficiently minimized in a way that is independent of the actual optimization procedure. Thus, \alg can simply be used as a preprocessing step before the actual optimization starts.

Incremental gradient methods aim at estimating the full gradient $\nabla f(w)$ over $V$ by iteratively making a step based on the gradient of every function $f_i$. 
%
Our key idea in \alg is that if we can find a small subset $S$ such that the weighted sum of
the gradients of its elements closely approximates the full gradient over $V$,
 we can apply IG only to the set $S$ (with stepsizes equal to the weight of the elements in $S$), and we should still converge to 
the (approximately) optimal solution, but much faster. 

Specifically, our goal in \alg is to find the smallest subset $S \subseteq V$ 
and corresponding per-element stepsizes $\gamma_j > 0$ that approximate the full gradient with an error at most $\epsilon > 0$ for all the possible values of the optimization parameters $w \in \mathcal{W}$.\footnote{Note that in the worst case we may need $|S^*| \approx |V|$ to approximate the gradient. However, as we show in experiments, in practice we find that 
a small subset is sufficient to accurately approximate the gradient.}
\begin{align}
\label{eq:error}
S^*=&{\arg\min}_{S \subseteq V, \gamma_j \geq 0~ \forall j} |S|, 
\;\; \text{s.t.}  \nonumber\\ 
&\max_{w\in\mathcal{W}}\| \sum_{i\in V} \nabla f_i(w) - \sum_{j \in S} \gamma_{j} \nabla f_{j}(w) \| \leq \epsilon. 
\end{align}
Given such an $S^*$ and associated weights $\{\gamma\}_j$, we are guaranteed that gradient updates on $S^*$ will be similar to the  gradient updates on $V$ regardless of the value of $w$.

Unfortunately, directly solving the above optimization problem is not feasible, due to two problems. Problem 1:  Eq.~(\ref{eq:error}) requires us to calculate the gradient of all the functions $f_i$ over the entire space $\mathcal{W}$, which is too expensive and would not lead to overall speedup. 
In other words, it would appear that solving for $S^*$ is as difficult as solving Eq.~(\ref{eq:problem}), as it involves calculating $\sum_{i \in V} \nabla f_i(w)$ for various $w \in \mathcal{W}$.
And Problem 2: even if calculating the normed difference between the gradients in Eq.~(\ref{eq:error}) would be fast, as we discuss later finding the optimal subset $S^*$ in NP-hard. 
In the following, we address the above two challenges and discuss how we can quickly find a near-optimal subset $S$.

\subsection{Upper-bound on the Estimation Error}\label{sec:upper}
We first address Problem 1, i.e., how to quickly estimate the error/discrepancy of the weighted sum of gradients of functions $f_j$ associate with data points $j \in S$, vs the full gradient, for every $w \in \mathcal{W}$.

Let $S$ be a subset of $r$ data points. 
Furthermore, assume that there is a mapping $\varsigma_w: V \rightarrow S$ that for every $w\in \mathcal{W}$ assigns every data point $i\in V$ 
to one of the elements $j$ in $S$, i.e., ${\varsigma_w(i)}={j} \in S$. Let $C_j = \{ i \in [n] | {\varsigma(i)} = {j} \} \subseteq V$ be the set of 
data points that are assigned to ${j} \in S$, and $\gamma_j = |C_j|$ be the number of such data points. 
Hence, $\{C_j\}_{j=1}^r$ form a partition of $V$.
Then, for any arbitrary (single) $w \in \mathcal{W}$ we can write 
\vspace{-1mm}
\begin{align}
\!\!\sum_{i\in V} \!\nabla f_i(w) 
&\!=\!  \sum_{i \in V}  \!\big( \nabla f_i(w) \!-\! \nabla f_{\varsigma_w(i)}(w) \!+\!\nabla f_{\varsigma_w(i)}(w) \big)\!\!\\
= \sum_{i\in V}  \big( &\nabla f_i(w) - \nabla f_{\varsigma_w(i)}(w) \big) + \sum_{j\in S} \gamma_j \nabla f_{j}(w).
\vspace{-2mm}
\end{align}
Subtracting and then taking the norm of the both sides, we get an
upper bound on the error of estimating the full gradient
 with the weighted sum of the gradients of the functions $\!f_{j}$ for $\!j \!\in \!S$. I.e.,
\begin{align}
\vspace{-2mm}
\| \sum_{i\in V} \nabla f_i(w) - \sum_{j\in S} &\gamma_j \nabla f_{j}(w) \| 
\leq \nonumber\\
&\sum_{i\in V}  \| \nabla f_i(w) - \nabla f_{\varsigma_w(i)}(w) \|,  \label{eq:upper}
\vspace{-2mm}
\end{align}
where the inequality follows from the triangle inequality. 
The upper-bound in Eq. (\ref{eq:upper}) is minimized when 
$\varsigma_w$ assigns every $i \in V$ to an element in $S$ with most gradient similarity at $w$, or minimum Euclidean distance between the gradient vectors at $w$. That is:
$\!\varsigma_w(i) \in {\arg\min}_{j \in S} \| \nabla f_i(w) - \nabla f_{j}(w) \|$. Hence,
\begin{align}
\hspace{-2mm}
\min_{S\subseteq V} \| \sum_{i\in V} \nabla f_i(w) - & \sum_{j\in S} \gamma_j \nabla f_{j}(w) \| 
\leq \nonumber\\
& \sum_{i\in V} \min_{j \in S} \| \nabla f_i(w) - \nabla f_{j}(w) \|.
\vspace{-2mm}\label{eq:min_upper}
\end{align}
The right hand side of Eq. (\ref{eq:min_upper}) is minimized when 
$S$ is the set of $r$ {\em medoids} (exemplars) \cite{kaufman1987clustering} for all the components in the gradient space. 
%

So far, we considered upper-bounding the gradient estimation error at a particular $w \in \mathcal{W}$. To bound the estimation error for all $w \in \mathcal{W}$, we consider a worst-case approximation of the estimation error over the entire parameter space $\mathcal{W}$. Formally, we define a distance metric $d_{ij}$ between gradients of $f_i$ and $f_j$ as the maximum normed difference between $\nabla f_i(w)$ and $\nabla f_j(w)$ over all $w \in \mathcal{W}$:
\begin{equation}
    d_{ij} \triangleq \max_{w \in \mathcal{W}} \| \nabla f_i(w) - \nabla f_{j}(w) \|.
\end{equation}
Thus, by solving the following minimization problem,  we obtain the smallest weighted subset $S^*$ that approximates the full gradient by an error of at most $\epsilon$ for all $w \in \mathcal{W}$:
\begin{eqnarray}\label{eq:fl_min}
\vspace{-1mm}
S^*\!\!=\!{\arg\min}_{S \subseteq V}|S|, \quad \text{s.t.} \quad L(S)\triangleq\sum_{i\in V} \min_{j \in S} d_{ij} \!\leq\! \epsilon.
\vspace{-1mm}
\end{eqnarray}
%
Note that Eq. (\ref{eq:fl_min}) requires that the gradient error is bounded over $\!\mathcal{W}$. However, we show (Appendix \ref{app:grad_bound}) for several classes of convex problems, including linear regression, ridge regression, logistic regression, and regularized support vector machines (SVMs), the normed gradient difference between data points can be efficiently boundedly approximated by \citep{allen2016exploiting, hofmann2015variance}:
\begin{align}\label{eq:upper_feat}
\vspace{-2mm}
    \forall  w,i,&j\quad \| \nabla f_{i}(w) - \nabla f_{j}(w) \| \leq d_{ij} \leq \nonumber\\
    &\max_{w \in \mathcal{W}} \mathcal{O}(\| w \|) \cdot \| x_i - x_j \| = \text{const.} ~ \| x_i - x_j \|.
\vspace{-1mm}
\end{align}
Note when $\| w \|$ is bounded for all $w \in \mathcal{W}$, i.e., $\max_{w \in \mathcal{W}} \mathcal{O}(\| w \|) < \infty$, upper-bounds on the Euclidean distances between the gradients can be pre-computed. This is crucial, because it means that estimation error of the full gradient can be efficiently bounded independent of the actual optimization problem (i.e., point $w$). Thus, these upper-bounds can be computed only once as a pre-processing step before any training takes place, and then used to find the subset $S$ by solving the optimization problem (\ref{eq:fl_min}). We address upper-bounding the normed difference between gradients for deep models in Section \ref{sec:deep_up}.


\subsection{The \alg Algorithm}\label{sec:alg}

Optimization problem~(\ref{eq:fl_min}) produces a subset $S$ of elements with their associated weights $\{\gamma\}_{j \in S}$ or per-element stepsizes that closely approximates the full gradient. 
Here, we show how to efficiently approximately solve the above optimization problem 
to find a near-optimal subset $S$.

The optimization problem (\ref{eq:fl_min}) is NP-hard as it involves calculating the value of $L(S)$ for all the $2^{|V|}$ subsets $S \subseteq V$. We show, however, that we can transform it into a {\em submodular set cover problem}, that can be efficiently approximated.

Formally, $F$ is submodular if $F(S\cup\{e\}) - F(S) \geq F(T\cup\{e\}) - F(T),$ for any $S\subseteq T \subseteq V$ and $e\in V\setminus T$. We denote the {\em marginal} utility of an element $e$ w.r.t. a subset $S$ as $F(e|S) =F(S\cup\{e\}) - F(S)$. Function $F$ is called \textit{monotone} if $F(e|S)\geq 0$ for any $e\!\in\! V\!\setminus \!S$ and $S\subseteq V$.
The submodular cover problem is defined as finding the smallest set $S$ that  achieves utility $\rho$. Precisely, 
\begin{equation} \label{problem1}
S^*={\arg\min}_{S \subseteq V} |S|, \quad \text{s. t.} \quad F(S) \geq \rho.
\vspace{-1mm}
\end{equation} 
Although finding $S^*$  is NP-hard since it captures such well-known NP-hard problems such as Minimum Vertex Cover, for many classes of submodular functions, a simple greedy algorithm is known to be very effective~\citep{nemhauser1978,wolsey1982analysis}.  The greedy algorithm starts with the empty set $S_0=\emptyset$, and at each iteration $i$, it  chooses an element $e\in V$ that maximizes $F(e|S_{i-1})$, i.e., 
$S_i = S_{i-1}\cup\{{\arg\max}_{e\in V} F(e|S_{i-1})\}.$ Greedy gives us a logarithmic approximation, i.e. $|S| \leq \big(1+ \ln (\max_e F(e|\emptyset))\big) |S^*|$ \cite{wolsey1982analysis}. The computational complexity of the greedy algorithm is $\mathcal{O}(|V|\cdot |S|)$. However, its running time can be reduced to $\mathcal{O}(|V|)$ using stochastic algorithms \citep{mirzasoleiman2015lazier} and further improved using lazy evaluation \citep{minoux1978accelerated}, and distributed implementations \citep{mirzasoleiman2015distributed, mirzasoleiman2016fast}.
Given a subset $S \subseteq V$, the facility location function quantifies the coverage of the whole data set $V$ by the subset $S$ by summing the similarities between every $i \in V$ and its closest element $j \in S$. Formally, facility location is defined as $F_{fl}(S)=\sum_{i \in V} \max_{j \in S} s_{i,j}$, where $s_{i,j}$ is the similarity between $i,j\in V$. The facility location function has been used in various 
summarization applications
\citep{lin2009-submodsum,lin2012learning}.
By introducing an auxiliary element $s_0$ we can turn $L(S)$ in Eq.~(\ref{eq:fl_min})  
into a monotone submodular facility location function,
\vspace{0mm}\begin{equation}
F(S)=L(\{s_0\})-L(S\cup \{s_0\}),
\vspace{-1mm}
\end{equation}
where $L(\{s_0\})$ is a constant. In words, $F$ measures the decrease in the estimation error associated with the set $S$ versus the estimation error associated
with just the auxiliary element. For a suitable choice of $s_0$, maximizing $F$ is equivalent to minimizing $L$.
Therefore, we apply the greedy algorithm to approximately solve the following problem to get the subset $S$ defined in Eq. (\ref{eq:fl_min}):
\begin{equation} \label{problem2}
S^*={\arg\min}_{S \subseteq V} |S|, \quad \text{s.t.} \quad F(S) \geq L(\{s_0\}) - \epsilon.
\end{equation} 
At every step, the greedy algorithm selects an element that reduces the upper bound on the estimation error the most. In fact, the size of the smallest subset $S$ that estimates the full gradient by an error of at most $\epsilon$ depends on the structural properties of the data. Intuitively, as long as the marginal gains of facility location are considerably large, we need more elements to improve our estimation of the full gradient.
Having found $S$, the weight $\gamma_j$ of every element $j \in S$ is the number of components that are closest to it in the gradient space, and are used as stepsize of element $j\in S$ during IG.
The pseudocode for \alg 
is outlined in Algorithm~\ref{alg:tree}.

Notice that \alg creates subset $S$ incrementally one element at a time, which produces a natural order 
to the elements in $S$. 
Adding the element with largest marginal gain $j\in {\arg\max}_{e \in V} F(e|S_{i-1})$
improves our estimation from the full gradient by an amount bounded by the marginal gain.   
At every step $i$, we have $F(S_{i}) \geq (1-e^{-i/|S|}) F(S^*)$.
Hence, for a greedily ordered subset $S=\{s_{1},\cdots,s_{k}\}$, we have
\vspace{-3mm}
\begin{equation}
    \|\sum_{i \in V} \nabla f_i(w) - \sum_{j=1}^k \gamma_{s_j} \nabla f_{s_j}(w) \| \leq c - (1-e^{-j/k})L(S^*),
\vspace{-1mm}
\end{equation}
where $c$ is a constant.
Intuitively, the first elements of the ordering contribute the most to provide a close approximation of the full gradient and the rest of the elements further refine the approximation. Hence, the first incremental gradient updates gets us close to $\opt$, and the rest of the updates further refine the solution.

\hide{
at every step we have
\begin{eqnarray}
\| \sum_v \nabla f_v - \sum_{\sigma_i \in S_j^{g}} n_i \nabla f_{\sigma_i} \|_2^2 
&\leq& \text{cnt}'' - (1-e^{-j/k}) \sum_{v \in V} \min_{\sigma_i \in S^*} \| \nabla f_v - \nabla f_{\sigma_i} \|_2^2. 
\end{eqnarray} 

The upper-bound in Problem~(\ref{eq:min_upper}) is related to the
submodular facility location function which has the form
$h(S) = \sum_{i=1}^n \max_{s \in S} w(s,i) = c - \sum_{i=1}^n \min_{s
	\in S} (c - w(s,i))$ for any constant $c$. If $c$ is large enough,
$h$ can be maximized,
subject to various constraints, approximately optimally using any of a
variety of greedy procedures \cite{nemhauser1978,wolsey1982analysis}.



By introducing an auxiliary element $f_{s_0}$ we can turn the minimization problem~(\ref{eq:fl_min}) into a monotone submodular cover problem, which can be efficiently approximated. 
\begin{eqnarray} \label{eq:flmax}
S^*&=&{\arg\min}_{S \subseteq\{f_1, \cdots, f_n\}} |S|,  \quad \text{such that} \hfill \nonumber \\
&& F_{fl}(f_{s_0}) - F_{fl}(S \cup \{f_{s_0}\} ) \geq F_{fl}(f_{s_0})  - \epsilon \quad\quad\quad \forall w \in \mathcal{W}.
\end{eqnarray} 
Technically, any element $f_{s_0}$ that satisfies  $\max_{}\| \nabla f_{s_0}(w) - \nabla f_i(w) \| \leq \| \nabla f_{s_0}(w) - \nabla f_i(w) \|, \forall i \in [n], \forall w \in \mathcal{W}$ can be used as an auxiliary element.
}
\hide{
	Let $S^{g}$ be the greedy solution to the above maximization problem, then
	\begin{eqnarray}
	\text{cnt} - \| \sum_v \nabla f_v - \sum_{\sigma_i \in S^{g}} n_i \nabla f_{\sigma_i} \|_2^2 
	&\geq& \text{cnt} - \sum_{v \in V} \min_{\sigma_i \in S^g} \| \nabla f_v - \nabla f_{\sigma_i} \|_2^2 \\
	&\geq&  (1-1/e)[\text{cnt} - \sum_{v \in V} \min_{\sigma_i \in S^*} \| \nabla f_v - \nabla f_{\sigma_i} \|_2^2]. 
	\end{eqnarray} 
	Therefore,
	\begin{eqnarray}
	\| \sum_v \nabla f_v - \sum_{\sigma_i \in S^{g}} n_i \nabla f_{\sigma_i} \|_2^2 
	&\leq& \text{cnt}' - (1-1/e) \sum_{v \in V} \min_{\sigma_i \in S^*} \| \nabla f_v - \nabla f_{\sigma_i} \|_2^2. 
	\end{eqnarray} 

Similarly, at every step we have
\begin{eqnarray}
\| \sum_v \nabla f_v - \sum_{\sigma_i \in S_j^{g}} n_i \nabla f_{\sigma_i} \|_2^2 
&\leq& \text{cnt}'' - (1-e^{-j/k}) \sum_{v \in V} \min_{\sigma_i \in S^*} \| \nabla f_v - \nabla f_{\sigma_i} \|_2^2. 
\end{eqnarray} 
}

\begin{algorithm}[t]
	\begin{algorithmic}[1]
		\REQUIRE Set of component functions $f_i$ for $i \in V=[n]\}$.
		\ENSURE Subset $S \subseteq V$ with corresponding per-element stepsizes $\{\gamma\}_{j\in S}$. 
		\STATE $S_0 \leftarrow \emptyset, s_0=0, i=0.$
		\WHILE {$F(S) < L(\{s_0\})  - \epsilon$}
		\STATE $j\in {\arg\max}_{e \in V\setminus S_{i-1}} F (e|S_{i-1})$
		\STATE $S_i = S_{i-1}\cup \{j\}$
		\STATE $i = i + 1$
		\ENDWHILE
		\FOR {$j=1$ to $|S|$}
		\STATE{$\gamma_j = \sum_{i\in V} \mathbb{I} \big[ j = {\arg\min}_{s \in S} {\max_{w \in \mathcal{W}}}  \| \nabla f_i(w) - \nabla f_{s}(w) \|  \big]$}
		\ENDFOR
	\end{algorithmic}
	\caption{\textsc{\alg} (CoResets for Accelerating Incremental Gradient descent) }
	\label{alg:tree}
\end{algorithm}
	\vspace{-1mm}
	
\subsection{\alg with Limited Budget}\label{sec:dual}
In practice, we often have a limited budget in terms of time or computational resources, and we are interested to find a near-optimal subset of size $r$ that best approximates the full gradient.
This problem can be formulated as a submodular maximization problem which is dual to the submodular cover problem (\ref{problem2}): 
\begin{equation} \label{dual}
S^*\in{\arg\max}_{S \subseteq V} F(S), \quad \text{s.t.} \quad |S| \leq r.
\end{equation}
For the above submodular maximization problem, the greedy algorithm discussed in Section \ref{sec:alg} provides a $(1-1/e)$-approximation to the optimal solution. 
For a subset $S$ of size at most $r$ obtained by the greedy algorithm, we can calculate the value of $\epsilon$ as follows:
\begin{equation}
    \epsilon \leq F(S) - L(\{s_0\}).
\end{equation}
We use this formulation in our experiments in Section \ref{sec:experiments}.

\subsection{Application of \alg to Deep Networks}\label{sec:deep_up}
As discussed, \alg selects a subset that closely approximates the full gradient, and hence can be also applied for speeding up training deep networks.
The challenge here is that we cannot use inequality (\ref{eq:upper_feat}) to bound the normed difference between gradients for all $w \in \mathcal{W}$ and find the subset as a preprocessing step.

However, for deep neural networks, the variation of the gradient norms is mostly captured by the gradient of the loss w.r.t. the input to the last layer [Section 3.2 of \cite{katharopoulos2018not}]. We show (Appendix \ref{app:grad_bound}) that the normed gradient difference between data points can be efficiently bounded approximately by
\begin{align}
\| \nabla f_i(w) - \nabla & f_j(w) \| \leq & \\ 
c_1\|\Sigma'_L(z_i^{(L)}) & \nabla f_i^{(L)}(w)- \Sigma'_L(z_j^{(L)}) \nabla f_j^{(L)}(w)\|+c_2, \nonumber
\end{align}
where $\Sigma'_L(z_i^{(L)})\nabla f_i^{(L)}(w)$ is gradient of the loss  w.r.t. the input to the last layer for data point $i$, and $c_1, c_2$ are constants.
The above upper-bound depends on parameter vector $w$ which changes during the training process. Thus, we need to use \alg to update the subset $S$ after a number of parameter updates.


%
The above upper-bound is often only slightly more expensive than calculating the loss. 
For example, in cases where we have cross entropy loss with soft-max as the last layer, the gradient of the loss w.r.t. the $i$-th input to the soft-max is simply $p_i-y_i$, where $p_i$ are logits (dimension $p\!-\!1$ for $p$ classes) and $y$ is the one-hot encoded label. In this case, \alg does not need any backward pass or extra storage. 
Note that, although \alg needs an additional $\mathcal{O}(|V|\cdot|S|)$ complexity (or $\mathcal{O}(|V|)$ using stochastic greedy) to find the subset $S$ at the beginning of every epoch, this complexity does not involve any (exact) gradient calculations and is negligible compared to the cost of backpropagations performed during the epoch. Hence, as we show in the experiments \alg is practical and scalable.


\section{Convergence Rate Analysis of \alg}
The idea of \alg is to selects a subset that closely approximates the full gradient, and hence can be applied to speed up most IG variants as we show in our experiments. Here, we briefly introduce the original IG method, and then prove the convergence rate of IG applied 
to \alg subsets.

\subsection{Incremental Gradient Methods (IG)}


Incremental gradient (IG) methods are core algorithms for solving
Problem~(\ref{eq:problem}) and are widely used and studied. IG aims at
approximating the standard gradient method by sequentially stepping
along the gradient of the component functions $f_i$ in a cyclic
order. Starting from an initial point $w_0^1 \in \mathbb{R}^d$, it
makes $k$ passes over all the $n$ components. At every epoch
$k \geq 1$, it iteratively updates $w_{i-1}^k$ based on the gradient of
$f_i$ for $i=1,\cdots,n$ using stepsize $\alpha_k > 0$. I.e.,
\begin{equation}\label{eq:update}
\vspace{-1mm}
w_{i}^k = w_{i-1}^k - \alpha_k \nabla f_i(w_{i-1}^k), \quad \quad i = 1,2,\cdots, n,
\end{equation}
with the convention that $w_0^{k+1} = w_n^k$. 
\hide{
Using the update relation~\ref{eq:update}, for each $k \geq 1$, we can write down the relation between the outer cycle as
\begin{equation}
w_0^{k+1}=w_0^k-\alpha_k \sum_{i=1}^n \nabla f_i(w_{i-1}^k),
\end{equation}
where $\sum_{i=1}^n \nabla f_i(w_{i-1}^k)$ is the aggregated component
gradients and serves as an approximation to the full gradient
$\nabla f(w_0^k) = \sum_{i=1}^n \nabla f_i (w_0^k)$.
}
%
Note that for a closed and convex subset $\mathcal{W}$ of $\mathbb{R}^d$, the results can be projected onto $\mathcal{W}$, and the update rule becomes 
\begin{equation}
w_{i}^k = P_{\mathcal{W}}(w_{i-1}^k - \alpha_k \nabla f_i(w_{i-1}^k)), \quad \quad i = 1,2,\cdots, n,
\end{equation}
where $P_\mathcal{W}$ denotes projection on the set $\mathcal{W} \subset \mathbb{R}^d$. 
\hide{
Moreover, for non-differentiable functions $f_i$, we can use subgradients and the method becomes incremental subgradient with update rule 
\begin{equation}
w_{i}^k = P_\mathcal{W}\big(w_{i-1}^k - \alpha_k g_{i}^k f_i(w_{i-1}^k) \big),
\end{equation}
where $g_{i}^k \in \partial f_i(w_{i-1}^k)$ and $\partial f_i(w_{i-1}^k)$ denotes the subdifferential (set of all subgradients) of $f_i$ at $w_{i-1}^k$.
}

IG with diminishing stepsizes converges at rate $\mathcal{O}(1/\sqrt{k})$ for strongly convex sum function \citep{nedic2001convergence}. If in addition to the strong convexity of the sum function, every component function $f_i$ is smooth, IG with diminishing stepsizes $\alpha_k=\Theta({1}/{k^s}), s\in(0,1]$ converges at rate $\mathcal{O}(1/k^s)$ \citep{gurbuzbalaban2015random}.

The convergence rate analysis of IG is valid regardless of order of processing the elements.
However, in practice, the convergence rate of IG is known to be quite sensitive to the order of processing the functions \citep{bertsekas2015convex,gurbuzbalaban2017convergence}. If problem-specific knowledge can be used to find a favorable order $\sigma$ (defined as a permutation $\{\sigma_1, \cdots, \sigma_n \}$ of $\{1,2,...,n\}$), IG can be updated to process the functions according to this order, i.e.,
\begin{equation}
w_{i}^k = w_{i-1}^k -
\alpha_k \nabla f_{\sigma_i}(w_{i-1}^k), \quad \quad i = 1,2,\cdots, n.
\end{equation}
In general a favorable order is not known in advance.
A common approach is sampling the function indices with replacement from the set $\{1, 2, \cdots , n\}$, which is called the
Stochastic Gradient Descent (SGD) method.


\subsection{Convergence Rate of IG on \alg Subsets}\label{sec:rate}
Next we analyze the convergence rate of IG applied to the weighted subset $S$ found by \alg. 
Note that \alg finds $S$ by greedily minimizing (\ref{problem2}) (or maximizing (\ref{dual})). Therefore, $S$ is a near-optimal solution of problem (\ref{eq:error}) and
estimates the full gradient by an error of at most $\epsilon$, i.e., $\max_{w\in\mathcal{W}}\| \sum_{i\in V} \nabla f_i(w) - \sum_{j \in S} \gamma_{j} \nabla f_{j}(w) \| \leq \epsilon$. 

Here, we show that (1) applying IG to $S$ converges to a close neighborhood of the optimal solution and that (2) this convergence happens at the same rate (same number of epochs) as IG on the full data. 
Formally, every step of IG on the subset becomes
\begin{align}
w_{i}^k = w_{i-1}^k - \alpha_k \gamma_{s_{\sigma_i}} \nabla &f_{s_{\sigma_i}}(w_{i-1}^k), \quad  i = 1,2,\cdots, r, \nonumber\\
&\quad\quad s_i \in S, \quad |S|=r.
\end{align}
Here, $\sigma$ is a permutation of $\{1,2,\cdots, r\}$, and the per-element stepsize $\gamma_{s_i}$ for every function $f_{s_i}$ is the weight of the element $s_i \in S$ and is fixed for all epochs. 

\subsection{Convergence for Strongly Convex \boldmath{$f$}}
We first provide the convergence analysis for the case where the 
function $f$ in Problem (\ref{eq:problem}) is strongly convex, i.e. $\forall w,w'\in \mathbb{R}^d$ we have $f(w) \geq f(w') + \langle \nabla f(w'), w-w' \rangle + \frac{\mu}{2} \| w'-w \|^2$. 

%

\begin{theorem} \label{thm:strong}
	Assume that $f$ is strongly convex, and $S$ is a weighted subset of size $r$ obtained by \alg 
	that estimates the full gradient by an error of at most $\epsilon$, i.e., $\max_{w\in\mathcal{W}}\| \sum_{i\in V} \nabla f_i(w) - \sum_{j \in S} \gamma_{j} \nabla f_{j}(w) \| \leq \epsilon$. 
	Then for the iterates $\{w_k\!=\!w_0^k\}$ generated by applying IG to $S$ with per-epoch stepsize $\alpha_k=\alpha/k^\tau$ with $\alpha>0$ and $\tau\in[0,1]$, we have 
	\begin{itemize}
		\item[(i)] if $\tau=1$, then $\| w_k - \opt \|^2  \!\leq\! 2\epsilon R/\mu^2 \!+ \alpha r^2\gamma_{\max}^2C^2\!/k\mu$,
		\item[(ii)] if $0< \tau < 1$, then $\| w_k - \opt \|^2 \!\leq\! 2\epsilon R/\mu^2$, \hfill for $k \rightarrow \infty$
		\item[(iii)] if $\tau=0$, then $\| w_k - \opt \|^2 \leq (1-\alpha \mu)^{k+1} \| w_{0} - \opt\|^2  + 2\epsilon R/\mu^2 + \alpha r^2 \gamma_{\max}^2 C^2/\mu,$
	\end{itemize}
	where $C$ is an upper-bound on the norm of the component function gradients, i.e. $\max_{i \in V} \sup_{w \in \mathcal{W}} \| \nabla f_i(w) \| \leq C$, $\gamma_{\max} = \max_{j \in S} \gamma_j$ is the largest per-element step size, and $R=\min\{d_0, (r\gamma_{\max}C + \epsilon)/\mu \}$, where $d_0=\| w_0 - w_*\|$ is the initial distance to the optimum $\opt$. 
\end{theorem}
All the proofs can be found in the Appendix.
The above theorem shows that IG on $S$ converges at the same rate $\mathcal{O}(1/\sqrt{k})$ of IG on the entire data set $V$. However, compared to IG on $V$, the $|V|/|S|$ speedup of IG on $S$ comes at the price of getting an extra error term,  $2\epsilon R/\mu^2$.


\begin{figure*}[t]
    \centering
    \subfloat[SGD \label{subfig:}]{
    \includegraphics[height=.35\textwidth,trim=10mm 0 10mm 0]{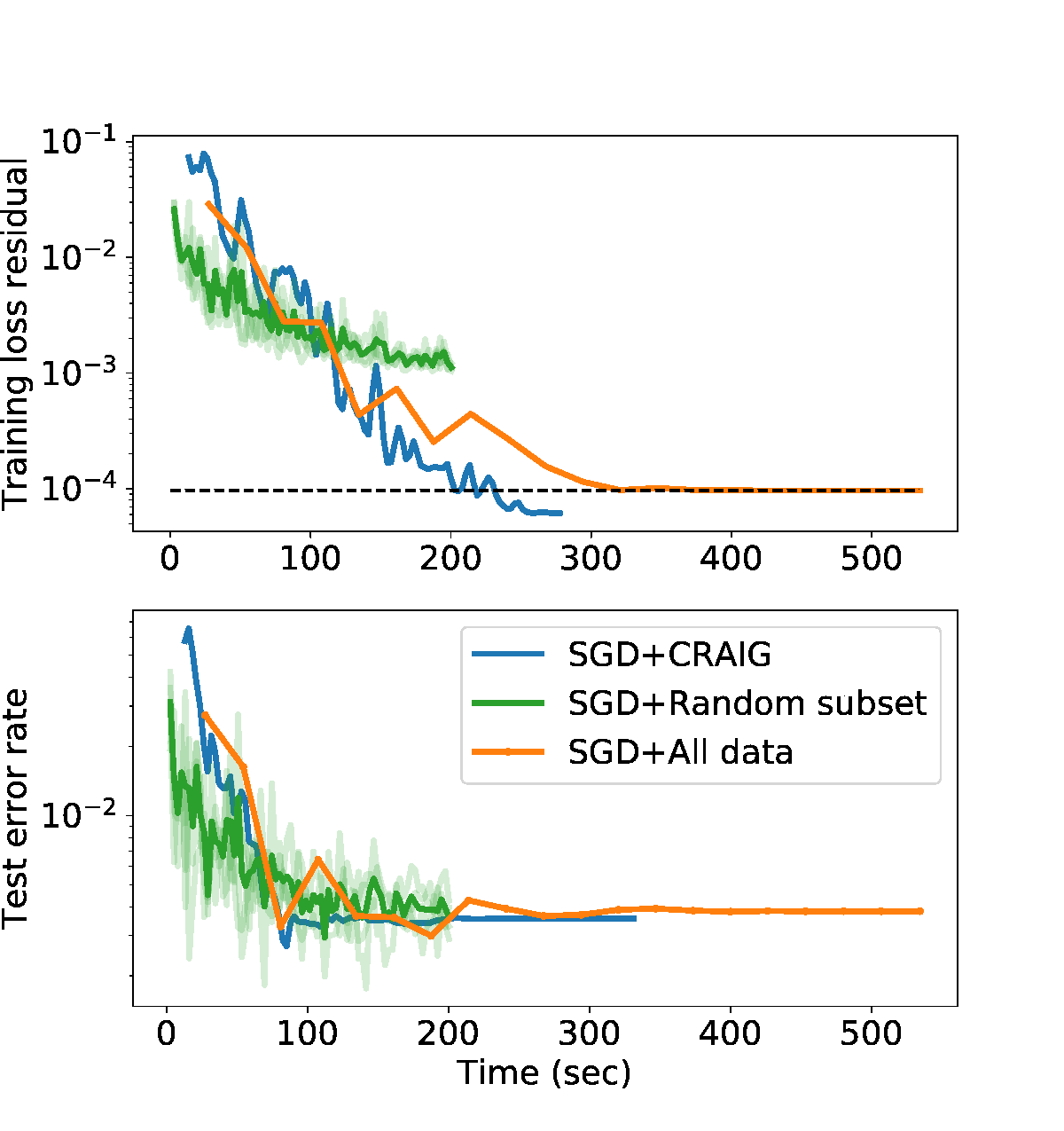}
    }\hspace{1mm}
    \subfloat[SAGA \label{subfig:}]{
	\includegraphics[height=.35\textwidth,trim=10mm 0 10mm 0]{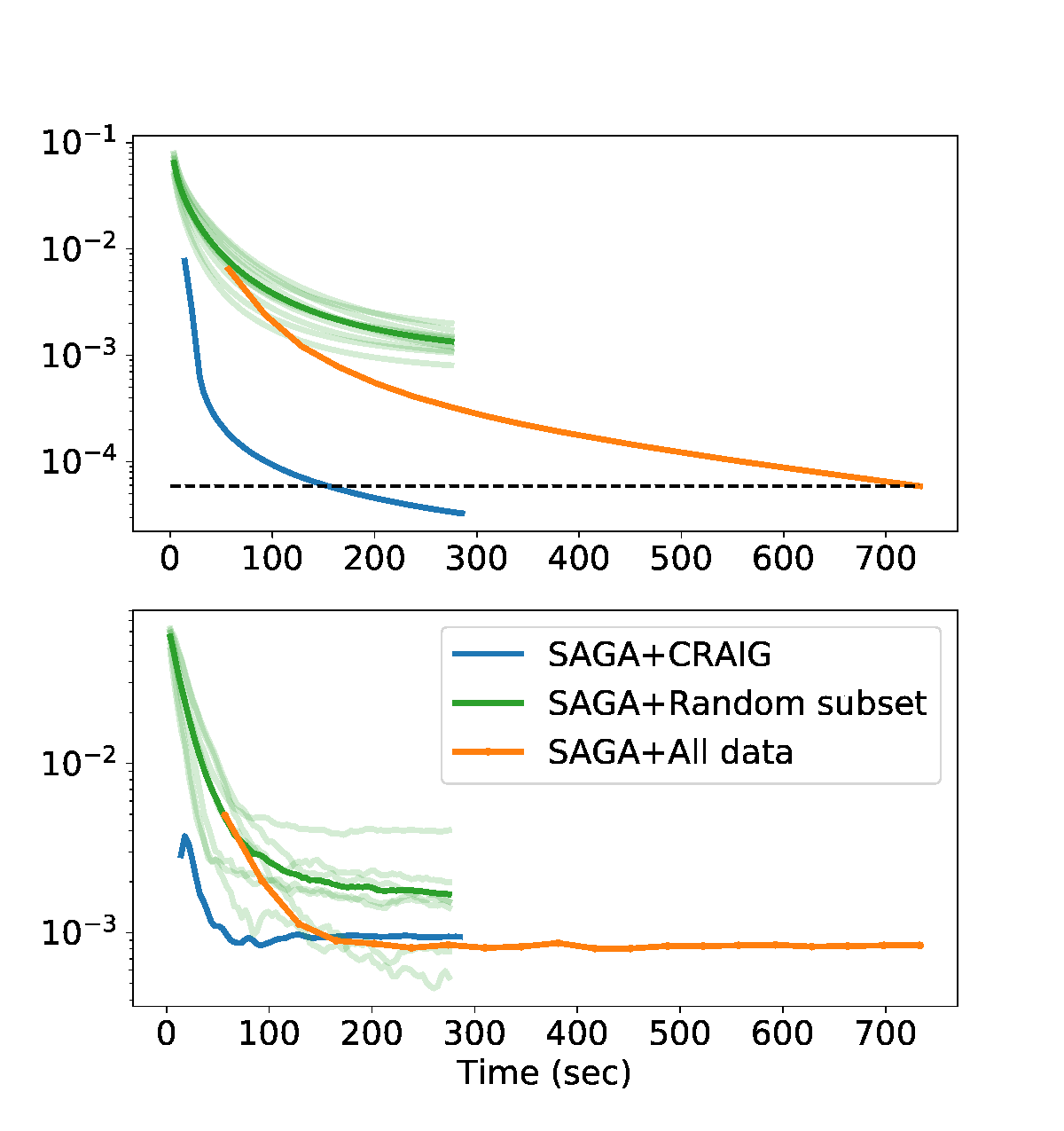}
	}\hspace{1mm}
    \subfloat[SVRG \label{subfig:}]{
	\includegraphics[height=.35\textwidth,trim=10mm 0 10mm 0]{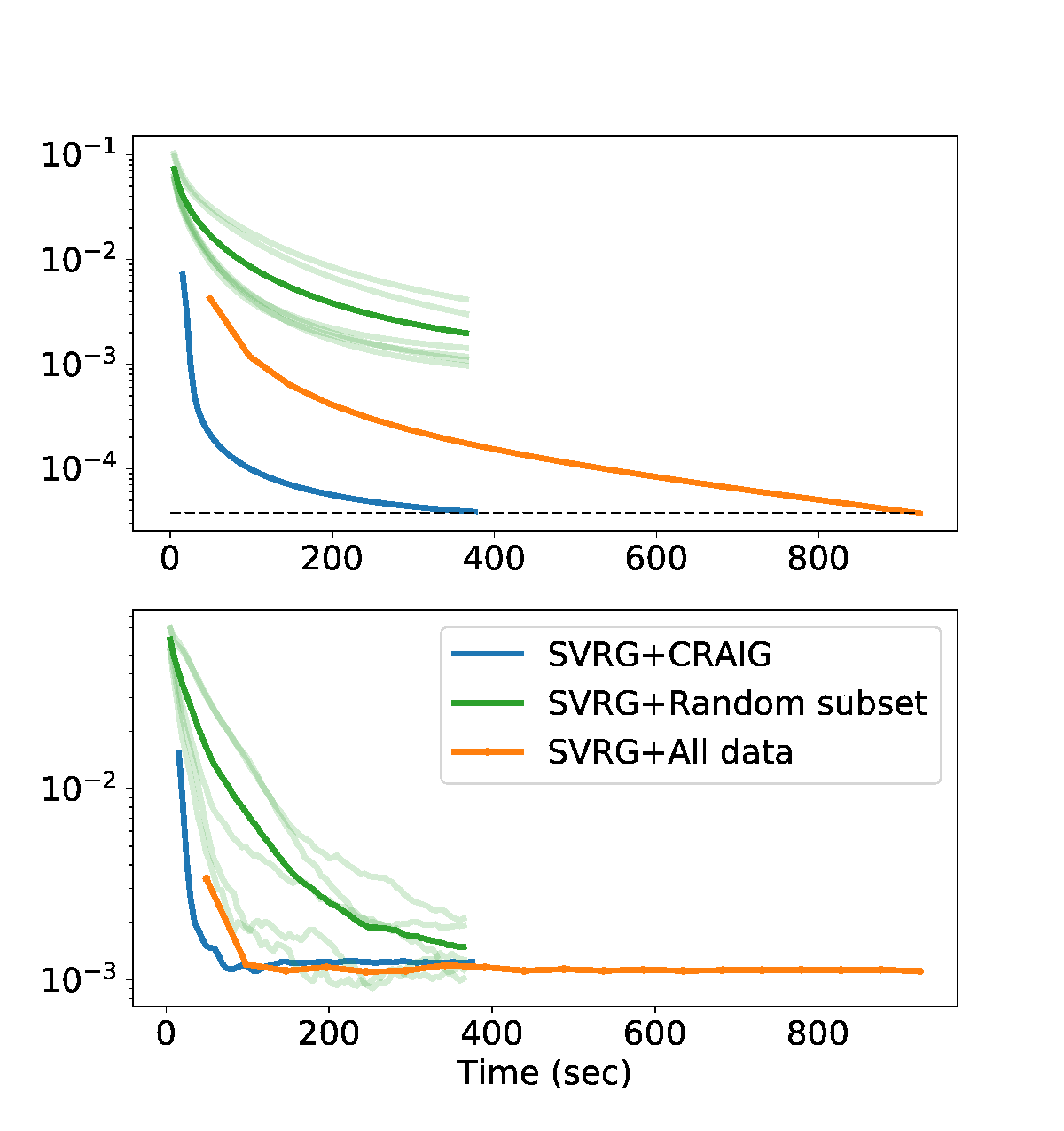}
	}
	\vspace{-3mm}
    \caption{{Loss residual and error rate of SGD, SVRG, SAGA for Logistic Regression on Covtype data set with 581,012 data points. We compare \alg (10\% selected subset) (blue) vs. 10\% random subset (green) vs. entire data set (orange). \alg gives the average speedup of 3x for achieving similar loss residual and error rate across the three optimization methods.}}
    \label{fig:var_sgd}
    \vspace{-3mm}
\end{figure*}

\subsection{Convergence for Smooth and Strongly Convex \boldmath{$f$}}
If in addition to strong convexity of the expected risk, each component function has a Lipschitz gradient, i.e. $ \forall w \in \mathcal{W},  i \in [n]$ we have $\| \nabla f_i(w) - \nabla f_i(w') \| \leq\! \beta_i \| w-w'\|$, then we get the following results about the iterates generated by applying IG to the weighted subset $S$ returned by \alg.

\begin{theorem}\label{thm:smooth}
	Assume that $f$ is strongly convex and let
	 $f_i(w), i=1,2,\cdots, n$ be convex and twice continuously differentiable component functions with Lipschitz gradients on $\mathcal{W}$. 
    Supposed that $S$ is a weighted subset of size $r$ obtained by \alg that estimates the full gradient by an error of at most $\epsilon$, i.e., $\max_{w\in\mathcal{W}}\| \sum_{i\in V} \nabla f_i(w) - \sum_{j \in S} \gamma_{j} \nabla f_{j}(w) \| \leq \epsilon$. 
	 Then for the iterates $\{w_k=w_0^k\}$ generated by applying IG to $S$ with per-epoch stepsize $\alpha_k=\alpha/k^\tau$ with $\alpha>0$ and $\tau\in[0,1]$, we have 
	\begin{itemize}
		\item[(i)] if $\tau=1$, then $\| w_k - \opt \|  \leq 2\epsilon/\mu + \alpha\beta Cr\gamma_{\max}^2/k\mu$,
		\item[(ii)] if $0< \tau < 1$, then $\| w_k - \opt \| \leq 2\epsilon/\mu$, \hfill for $k \rightarrow \infty$
		\item[(iii)] if $\tau=0$, then $\| w_k - \opt \|  \leq (1-\alpha \mu)^k \| w_0 - \opt \| + 2\epsilon/\mu + \alpha \beta C r \gamma_{\max}^2 / \mu$,
	\end{itemize}
	where $\beta = \sum_{i=1}^n \beta_i$ is the sum of gradient Lipschitz constants of the component functions.
\end{theorem}
The above theorem shows that for $\tau>0$, IG applied to $S$ converges to a $2\epsilon/\mu$ neighborhood of the optimal solution, with a rate of $\mathcal{O}(1/k^\tau)$ which is the same convergence rate for IG on the entire data set $V$.
As shown in our experiments, in real data sets small weighted subsets constructed by \alg provide a close approximation to the full gradient. Hence, applying IG to the weighted subsets returned by \alg provides a solution of the same or higher quality compared to the solution obtained by applying IG to the whole data set, in a considerably shorter amount of time.


\section{Experiments}\label{sec:experiments}

In our experimental evaluation we wish to address the following questions: (1) How do loss and accuracy of IG applied to the subsets returned by \alg compare to loss and accuracy of IG applied to the entire data? (2) How small is the size of the subsets that we can select with \alg and still get a comparable performance to IG applied to the entire data? 
And (3) How well does \alg scale to large data sets, and extends to non-convex problems? 
%
In our experiments, we report the run-time as the wall-clock time for subset selection with \alg, plus minimizing the loss using IG or other optimizers with the specified learning rates. For the classification problems, we separately select subsets from each class while maintaining the class ratios in the whole data, and apply IG to the union of the subsets. 
{
Note that the upper bounds on the gradient differences derived in Appendix \ref{app:grad_bound}
only hold for points with similar labels. Thus, theoretically we need to select subsets separately. For neural networks, we observed that separately selecting subsets from each class helps the performance.
}
We separately tune each method so that it performs at its best.


\subsection{Convex Experiments} \vspace{-2mm}
In our convex experiments, we apply \alg to SGD, as well as SVRG \citep{johnson2013accelerating}, and SAGA \citep{defazio2014saga}.
We apply L2-regularized logistic regression: $f_i(x) = \ln(1 + \exp(-w^T x_i y_i)) + 0.5 \lambda w^Tw$ to classify the following two datasets from LIBSVM:
(1) {\em covtype.binary} including 581,012 data points of 54 dimensions, and 
(2) {\em Ijcnn1} including 49,990 training and 91,701 test data points of 22 dimensions.
As {\em covtype} does not come with labeled test data, we randomly split the training data into halves to make the training/test split (training and set sets are consistent for different methods). 

For the convex experiments, 
we tuned the learning rate for each method (including the random baseline) by preferring smaller training loss from a large number of parameter combinations for two types of learning scheduling: exponential decay $\alpha_k = \alpha_0 b^{k}$ and $k$-inverse $\alpha_k = \alpha_0 (1 + bk)^{-1}$ with parameters $\alpha_0$ and $b$ to adjust.
{
For convergence of IG to $2\epsilon/\mu$ neighborhood of the optimal solution, we require that $\sum_{k=0}^{\infty} \alpha_k=\infty$, and $\sum_{k=0}^{\infty} \alpha^2_k=0$ \cite{nedic2001convergence}. Hence, while the convergence of exponentially decaying learning rate is not theoretically guaranteed, it often worked better in our experiments.
}
Furthermore, following \cite{johnson2013accelerating} we set $\lambda$  to $10^{-5}$.

\begin{figure}
    \centering
	\vspace{-5mm}
    \subfloat[Covtype \label{subfig:covtype_eps}]{
	\includegraphics[width=.24\textwidth,trim=10mm 0 10mm 0]{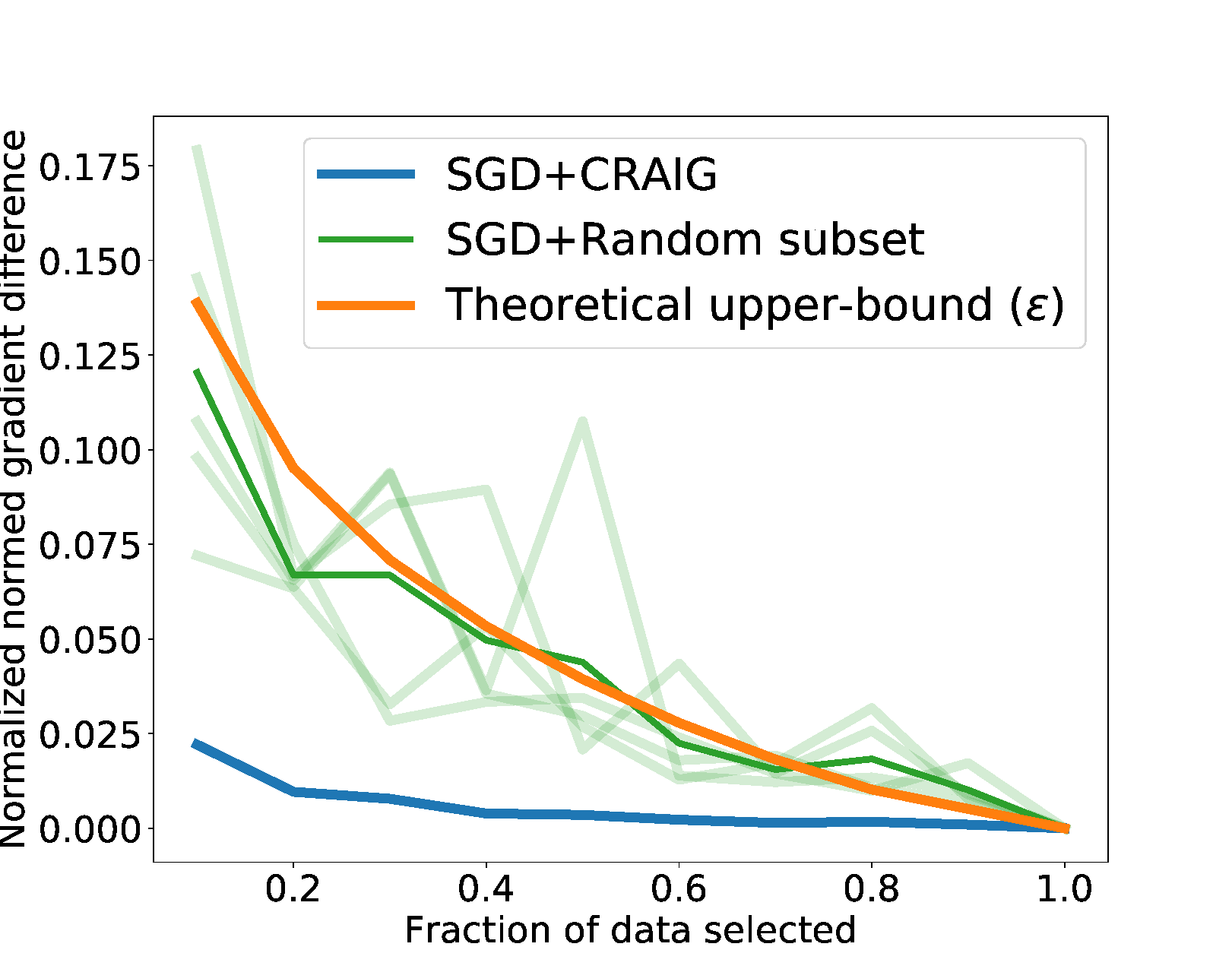}
    }
    \subfloat[Ijcnn1 \label{subfig:ijcnn1_eps}]{
    \includegraphics[width=.24\textwidth,trim=10mm 0 10mm 0]{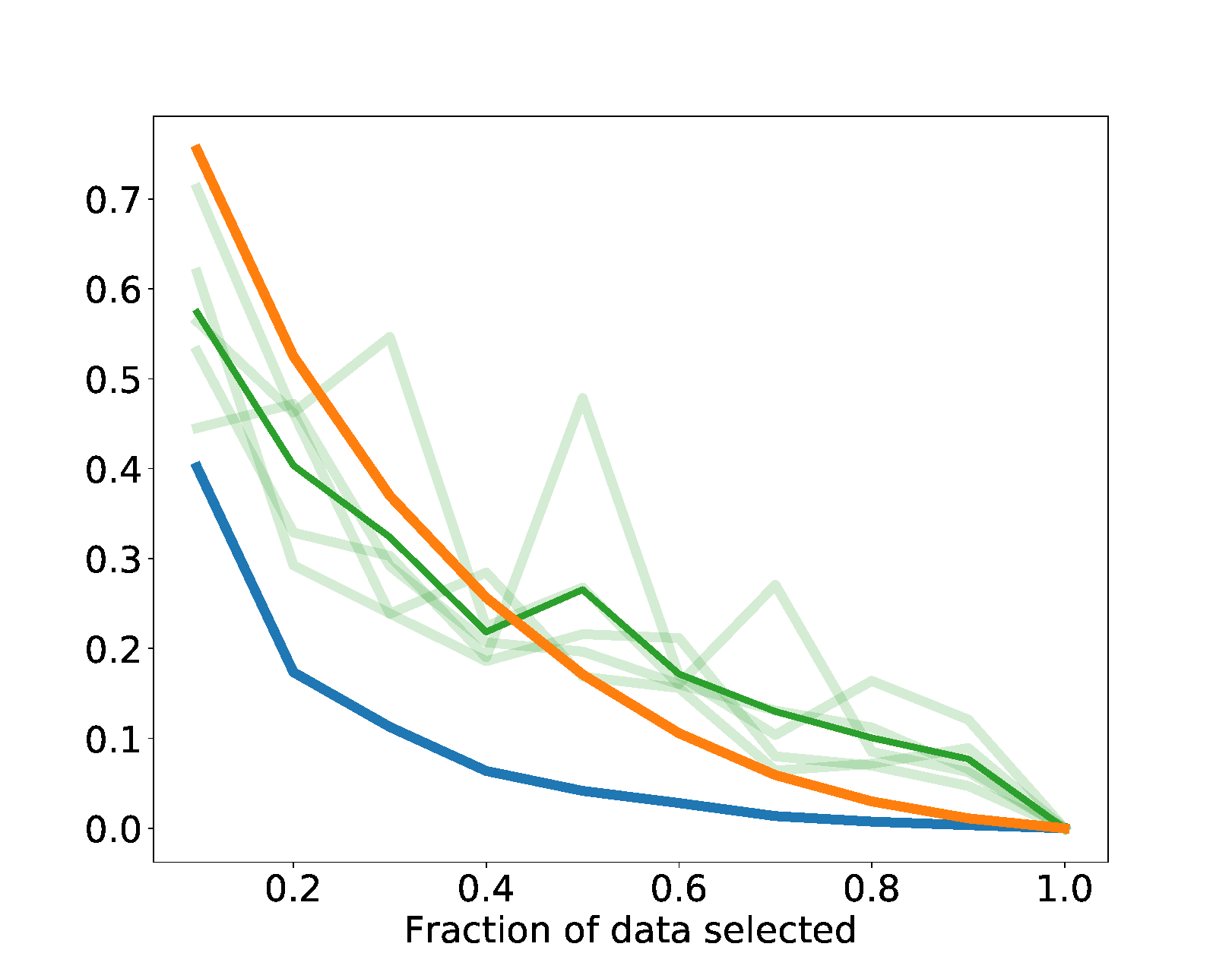}
    }
    \caption{Normed difference between the full gradient, the gradient of the subset found by \alg (Eq. \ref{eq:error}), and the theoretical upper-bound $\epsilon$ (Eq. \ref{eq:fl_min}). The values are normalized by the largest full gradient norm. The transparent green lines demonstrate various random subsets, and the opaque green line shows their average.
    }
    \vspace{-5mm}
    \label{fig:grad_diff}
\end{figure}
\begin{figure}
    \centering
    \vspace{-2mm}
    \includegraphics[width=.28\textwidth,trim=10mm 0 10mm 0]{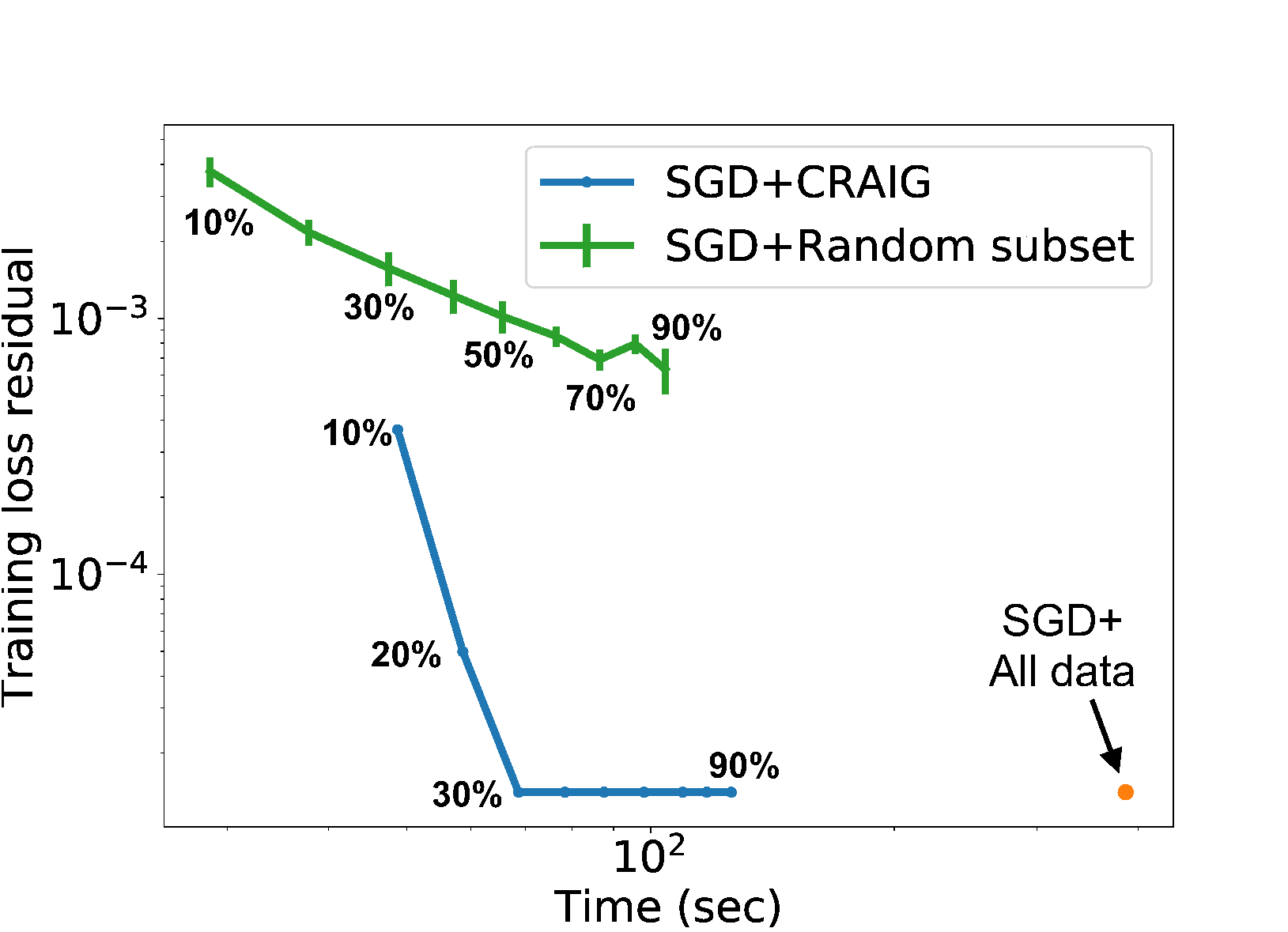}
    \caption{{
    Training loss residual for SGD applied to subsets of size $10\%, 20\%, \cdots, 90\%$ found by \alg vs. random subsets of the same size from Ijcnn1. We get 5.6x speedup from applying SGD to subset of size 30\% compared to the entire dataset.
    }}
    \label{fig:speedup}
    \vspace{-4mm}
\end{figure}

\hide{
\begin{figure}
    \centering
    \includegraphics[width=.5\textwidth]{Fig/mnist_rebut.png}
    \caption{Training loss and test accuracy for SGD applied to full MNIST vs. subsets of size 60\% selected by \alg and random subsets of size 60\%. Both the random subsets and the subsets found by \alg change at the beginning of every epoch.}
    \label{fig:my_label}
\end{figure}
}

\hide{
\begin{figure}
    \centering
	\subfloat[MNIST \label{subfig:mnist}]{
    \includegraphics[width=.1475\textwidth,trim=10mm 0 10mm 0]{}}\hspace{8mm}
    \subfloat[CIFAR10 \label{subfig:cifar}]{
	\includegraphics[width=.73\textwidth,trim=10mm 0 10mm 0]{}
	}
    \caption{{Training loss and test accuracy of \alg applied to (a) SGD on MNIST with a 1-layer neural network, and (b) SGD, Adam, Adagrad, NAG, on CIFAR-10 with ResNet-56.
    \alg provides 2x to 3x speedup and a better generalization performance.}
    \vspace{-1mm}}
    \label{fig:cifar_mnist}
\end{figure}

}

\xhdr{\alg effectively minimizes the loss}
Figure~\ref{fig:var_sgd}(top) compares training loss residual of SGD, SVRG, and SAGA on the 10\% \alg set (blue), 10\% random set (green), and the full dataset (orange). \alg effectively minimizes the training data loss (blue line) and achieves the same minimum as the entire dataset training (orange line) but much faster.
Also notice that training on the random 10\% subset of the data does not effectively minimize the training loss.


\begin{figure}
	\centering
	\includegraphics[width=.46\textwidth,trim=10mm 0 10mm 0]{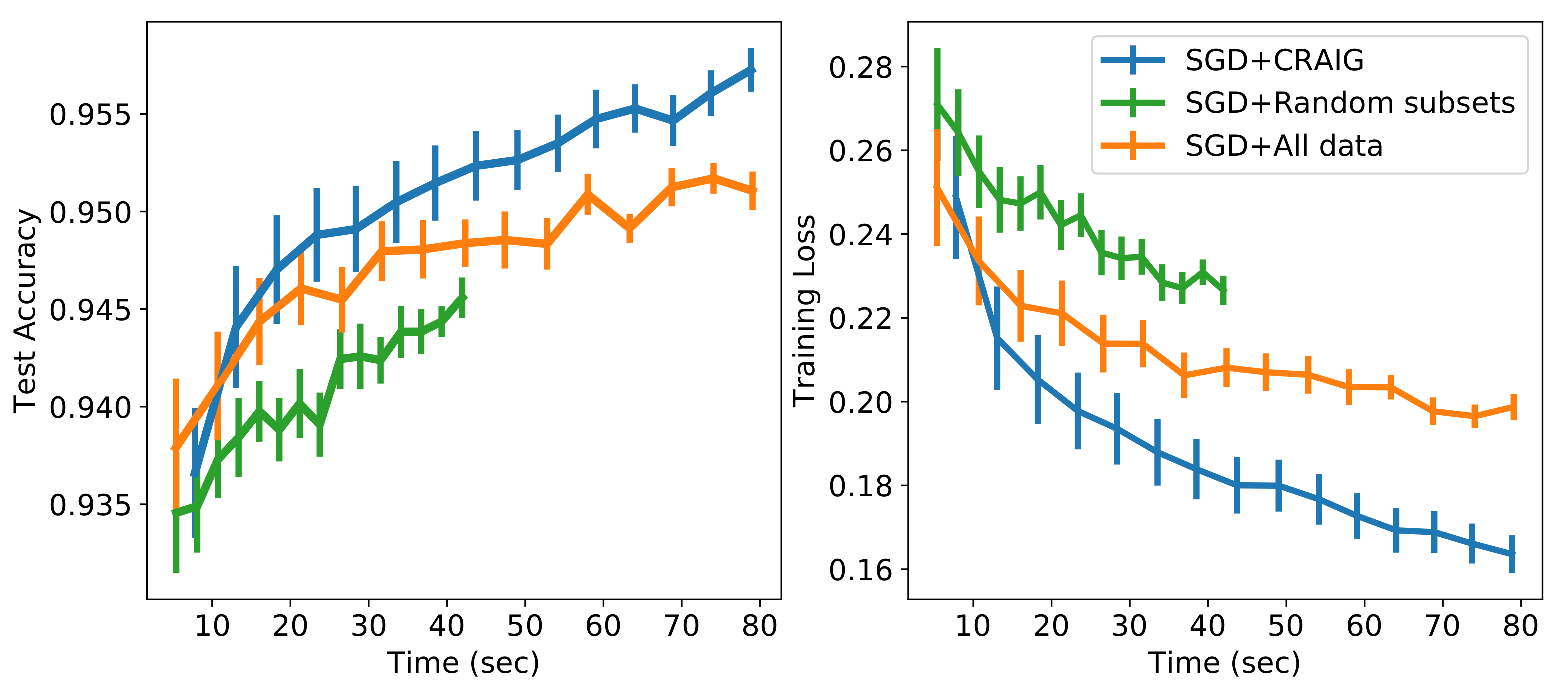}
	\caption{{Test accuracy and training loss of SGD applied to subsets found by \alg vs. random subsets on MNIST with a 2-layer neural network.
			\alg provides 2x to 3x speedup and a better generalization performance. 
	}}
	\vspace{-3mm}
	\label{fig:mnist}
\end{figure}  

\xhdr{\alg has a good generalization performance}
Figure~\ref{fig:var_sgd}(bottom) shows the test error rate of models trained on \alg vs. random vs. the full data. Notice that training on \alg subsets achieves the same generalization performance (test error rate) as training on the full data.

\xhdr{\alg achieves significant speedup}
Figure~\ref{fig:var_sgd} also shows that \alg achieves a similar training loss (top) and test error rate (bottom) as training on the entire set, but much faster. In particular, we obtain a speedup of 2.75x, 4.5x, 2.5x from applying IG, SVRG and SAGA on the subsets of size 10\% from covtype obtained by \alg.
Furthermore, Figure \ref{fig:speedup} compares the speedup achieved by \alg to reach a similar loss residual as that of SGD 
for subsets of size $10\%, 20\%, \cdots, 90\%$ of Ijcnn1. We get a 5.6x speedup by applying SGD to subsets of size 30\% obtained by \alg.

\xhdr{\alg gradient approximation is accurate}
Figure~\ref{fig:grad_diff} demonstrates the norm of the difference between the weighted gradient of the subset found by \alg and the full gradient compared to the theoretical upper-bound $\epsilon$ specified in Eq. (\ref{eq:fl_min}). 
The gradient difference is calculated by sampling the full gradient at various points in the parameter space. 
Gradient differences are then normalized by the largest norm of the sampled full gradients.
The figure also compares the normed gradient difference between gradients of several random subsets $S$ where each data point is weighted by $|V|/|S|$. Notice that \alg's gradient estimate is more accurate than the gradient estimate obtained by the same-size random subset of points (which is how methods like SGD approximate the gradient). This demonstrates that our gradient approximation in Eq. (\ref{eq:fl_min}) is reliable in practice.

\hide{
\xhdr{\alg benefits from ordering}
Finally, Figure \ref{fig:order} shows the loss residual vs time for IG when it processes the elements of the subsets according to the ordering obtained by \alg compared to random permutations of the same subsets. We observe that the greedy ordering significantly improves the rate of convergence of IG.
}

\subsection{Non-convex Experiments}\vspace{-2mm}
Our non-convex experiments involve applying \alg to train the following two neural networks: 
(1) Our smaller network is a fully-connected hidden layer of 100 nodes and ten softmax output nodes; sigmoid activation and L2 regularization with  $\lambda= 10^{-4}$ and mini-batches of size 10 on MNIST dataset of handwritten digits containing 60,000 training and 10,000 test images. 
(2) Our large neural network is ResNet-20 for CIFAR10 with convolution, average pooling and dense layers with softmax outputs and L2 regularization with $\lambda= 10^{-4}$. CIFAR 10 includes 50,000 training and 10,000 test images from 10 classes, and we used mini-batches of size 128. 
Both MNIST and CIFAR10 data sets are normalized into [0, 1] by division with 255.
In all these experiments, we report average test accuracy across 10 trials.

\begin{figure}
\centering    
    \subfloat[\label{subfig:lag1}]{
	\includegraphics[width=.22\textwidth,trim=10mm 0 10mm 0]{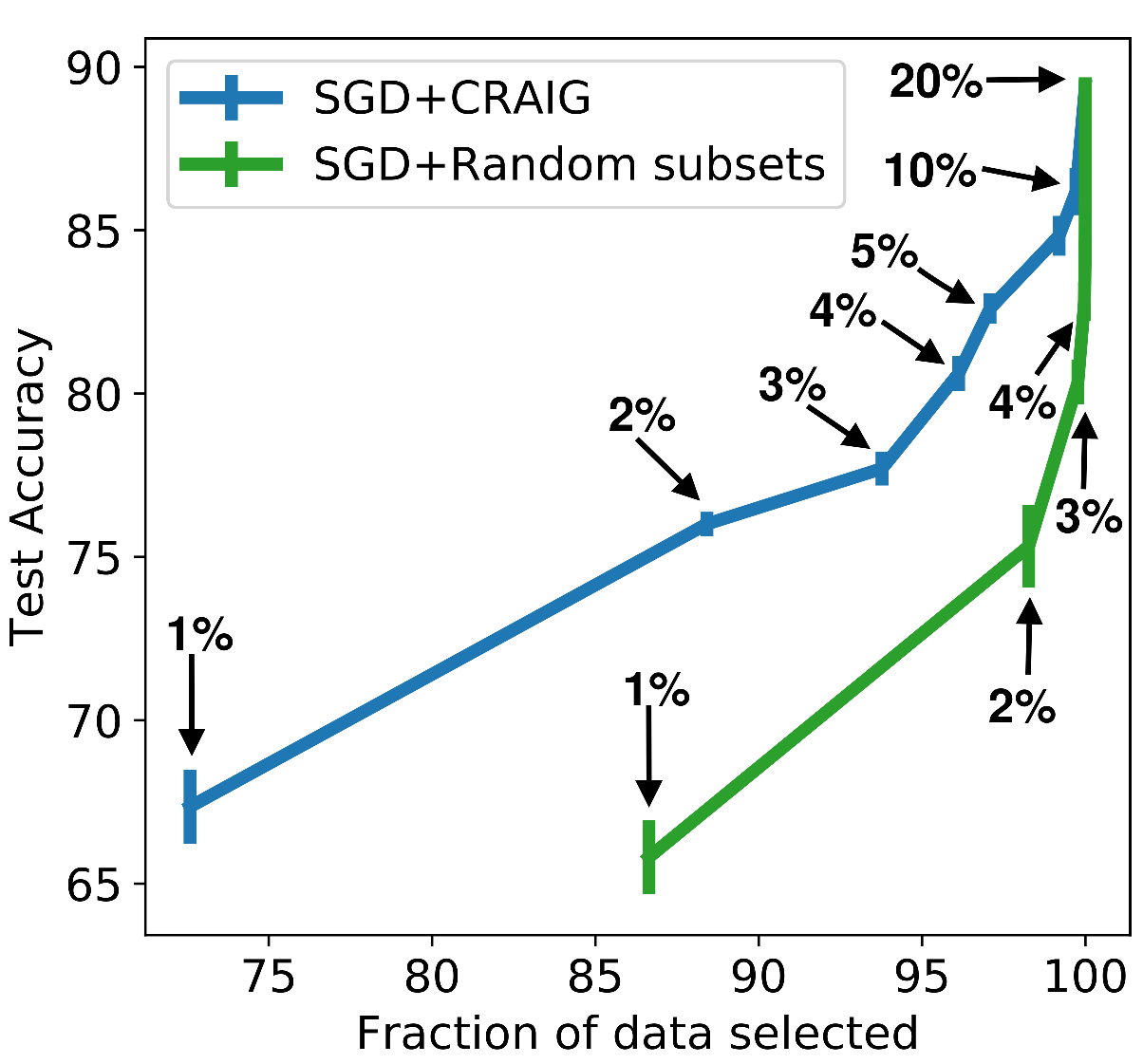}}\hspace{5mm}
	\subfloat[\label{subfig:lag5}]{
	\includegraphics[width=.21\textwidth,trim=10mm 0 10mm 0]{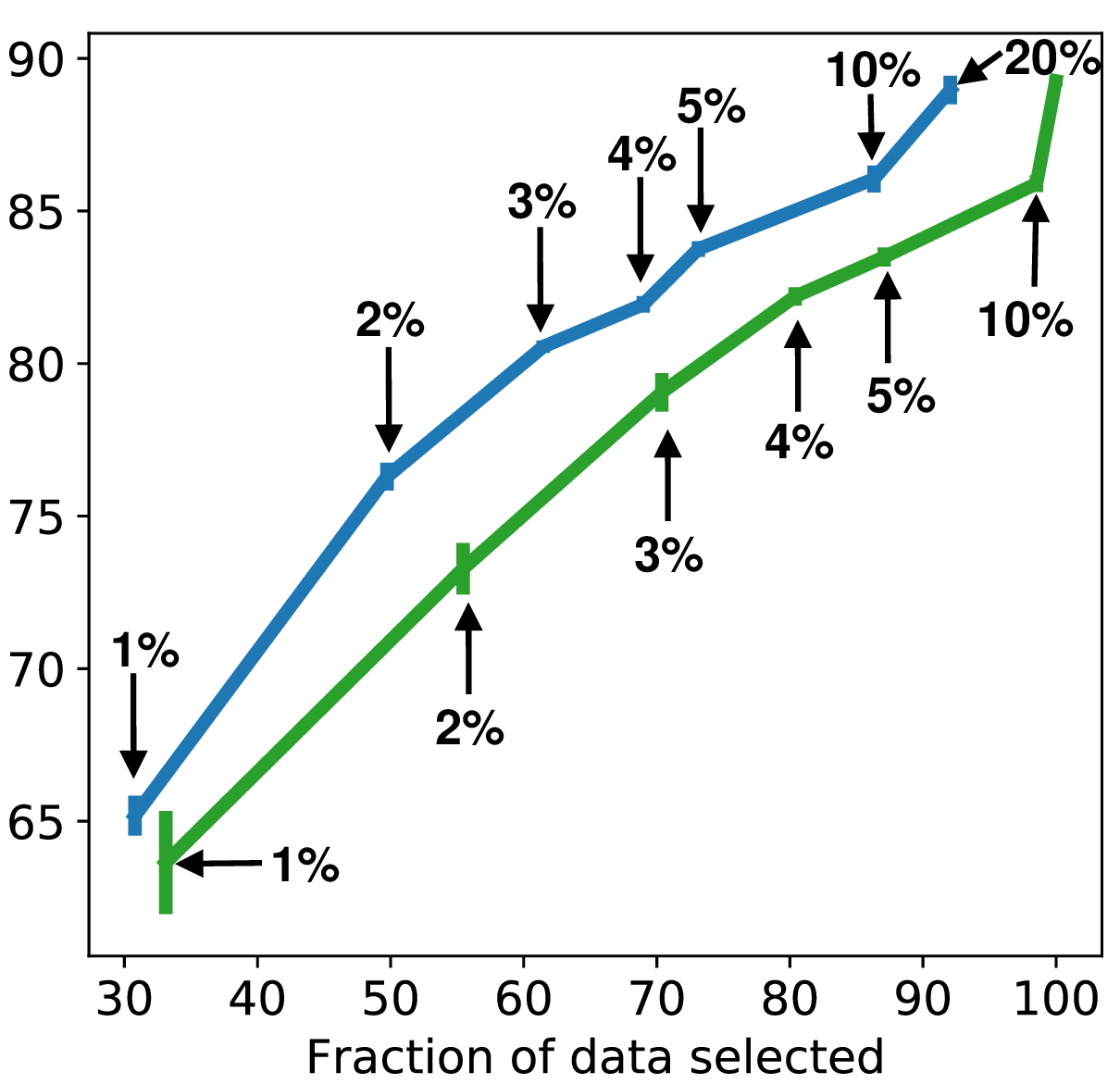}}
  \vspace{-1mm}
  \caption{{Test accuracy vs. fraction of data selected during training of ResNet-20 on CIFAR10. 
  (a) At the beginning of ever epoch, a new subset of size 1\%, 2\%, 3\%, 4\%, 5\%, 10\%, or 20\% is selected by \alg. 
  (b) Every 5 epochs a new subset of similar size is selected by \alg. 
  SGD 
  is then applied to training on the selected subsets. 
  The x-axis shows the fraction of training data points that are used by SGD 
  during the training process. 
  Note that for a given subset size, backpropagation is done on the same number of data points for \alg and random. However, \alg selects a smaller number of distinct data points during the training. Therefore, \alg is data-efficient for training neural networks.
  }}
  \label{fig:cifar}
\vspace{-3mm}
\end{figure}  

\xhdr{\alg achieves considerable speedup}
Figure~\ref{fig:mnist} shows training loss and test accuracy for training a 2-layer neural net on MNIST.  For this problem, we used a constant learning rate of $10^{-2}$. 
Here, we apply \alg to select a subset of 50\% of the data at the beginning of every epoch and train only on the selected subset with the corresponding per-element stepsizes.
Interestingly, in addition to achieving a speedup of 2x to 3x for training the network, the subsets selected by \alg provide a better generalization performance compared to models trained on the entire dataset.

\xhdr{\alg is data-efficient for training neural networks}
Figure~\ref{fig:cifar} shows test accuracy vs. the fraction of data points selected for training ResNet-20 on CIFAR10.
We trained the network for 200 epochs, and used the standard learning rate schedule for training ResNet-20 on CIFAR10. I.e., we start with initial learning rate of 0.1, and exponentially decay the learning rate by a factor of 0.1 at epochs 100 and 150. 
To prevent weights from diverging when training with subsets of all sizes, we used linear learning rate warm-up for 20 epochs from 0.
For optimization we used SGD with a momentum of 0.9. 

\begin{figure}
    \centering
    \subfloat[First\label{subfig:epoch1}]{
    \includegraphics[height=.36\textwidth,trim=10mm 0 10mm 0]{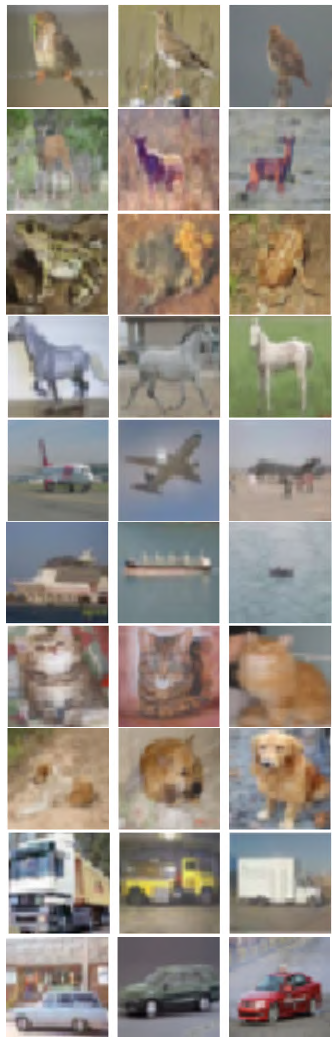}}\hspace{1cm}
    \subfloat[Middle \label{subfig:epoch100}]{
    \includegraphics[height=.36\textwidth,trim=10mm 0 10mm 0]{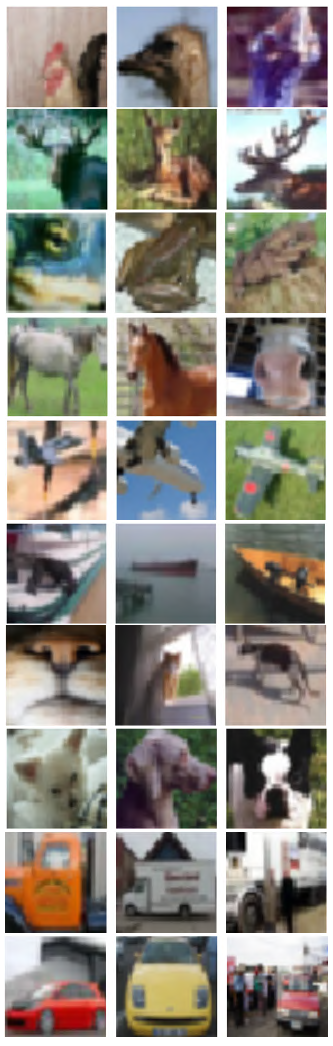}}\hspace{1cm}
    \subfloat[Last \label{subfig:epoch200}]{
    \includegraphics[height=.36\textwidth,trim=10mm 0 10mm 0]{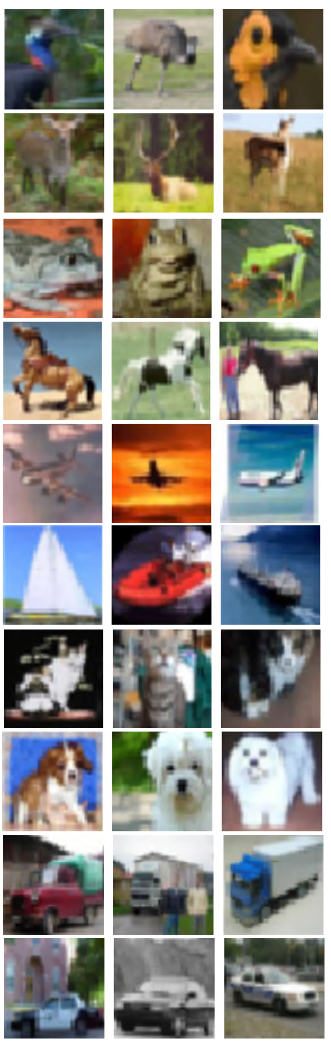}}
    \vspace{-3mm}
    \caption{A subset of images selected by \alg from CIFAR10. Subsets are selected at the (a) beginning of training (epoch 1), (b) middle of training (epoch 100), and (c) end of training (epoch 200). 
    We notice that during the training, the semantic redundancies decrease considerably, and coreset images better represent various types of images (that are more difficult to learn) in every class. 
    }
    \vspace{-3mm}
    \label{fig:semantic}
\end{figure}

Figure \ref{subfig:lag1} shows the test accuracy when at the beginning of every epoch a subset of size 1\%, 2\%, 3\%, 4\%, 5\%, 10\%, or 20\% is chosen at random or by \alg from the training data. 
The network is trained 
only on the selected subset of a given size for that epoch. 
For every subset size, the x-axis shows the fraction of training data points that are used during the entire training process.
Figure \ref{subfig:lag5} shows the test accuracy when a subset of size 1\%, 2\%, 3\%, 4\%, 5\%, 10\%, or 20\% is chosen at random or by \alg every 5 epochs. The network is trained only on the selected subset for 5 epochs.
Generally, larger subsets or more frequent updates lead to more data exposure and hence better performance.
However, since in deep networks the gradients may change rapidly after a small number of parameter updates \cite{defazio2019ineffectiveness}, selecting smaller subsets and more frequent updates result in a larger improvement over the random baseline.
Note that for a given subset size, backpropagation is done on the same number of data points for \alg and random.
However, it can be seen that \alg can identify the data points that are effective for training the neural network, and hence achieves a superior test accuracy by training on a smaller fraction of the training data. 

\xhdr{Insights from \alg subsets}
Figure~\ref{fig:semantic} shows a subset of images selected by \alg for training CIFAR10 at the beginning (\ref{subfig:epoch1}), middle (\ref{subfig:epoch100}), and end (\ref{subfig:epoch1}) of training. 
Since gradients are more uniformly distributed at initialization, subsets contain semantic redundancies at the beginning of the training (\ref{subfig:epoch1}). 
We notice that during the training, semantic redundancies decrease considerably.
In particular, as training proceeds coreset images represent groups of images that are more difficult to learn, e.g., contain parts of an object (\ref{subfig:epoch100}), or have a different foreground/background color than the rest of the images in a class (\ref{subfig:epoch200}).

\section{Conclusion}\vspace{-2mm}
We developed a method, \alg, for selecting a subset (coreset) of data points with their corresponding per-element stepsizes to speed up iterative gradient (IG) methods. In particular, we showed that weighted subsets that minimize the upper-bound on the estimation error of the full gradient, maximize a submodular facility location function.
Hence, the subset can be found
using a fast greedy algorithm. We proved that IG on subsets $S$ found by \alg converges at the same rate as IG on the entire dataset $V$, while providing a $|V|/|S|$ speedup. In our experiments, we showed that various IG methods, including SGD, SAGA, and SVRG 
runs up to 6x faster on convex and up to 3x on non-convex problems on subsets found by \alg while achieving practically the same training loss and test error. 

\section*{Acknowledgement}
This work was supported in part by the SNSF P2EZP2\_172187, and the CONIX Research Center, one of six centers in JUMP, a Semiconductor Research Corporation (SRC) program sponsored by DARPA.
We also gratefully acknowledge the support of
DARPA under Nos. FA865018C7880 (ASED), N660011924033 (MCS);
ARO under Nos. W911NF-16-1-0342 (MURI), W911NF-16-1-0171 (DURIP);
NSF under Nos. OAC-1835598 (CINES), OAC-1934578 (HDR), CCF-1918940 (Expeditions), IIS-2030477 (RAPID);
Stanford Data Science Initiative, 
Wu Tsai Neurosciences Institute,
Chan Zuckerberg Biohub,
Amazon, Boeing, Chase, Docomo, Hitachi, Huawei, JD.com, NVIDIA, Dell. 
J. L. is a Chan Zuckerberg Biohub investigator.

\bibliography{fast}

\begin{thebibliography}{49}
\providecommand{\natexlab}[1]{#1}
\providecommand{\url}[1]{\texttt{#1}}
\expandafter\ifx\csname urlstyle\endcsname\relax
  \providecommand{\doi}[1]{doi: #1}\else
  \providecommand{\doi}{doi: \begingroup \urlstyle{rm}\Url}\fi

\bibitem[Agarwal et~al.(2004)Agarwal, Har-Peled, and
  Varadarajan]{agarwal2004approximating}
Agarwal, P.~K., Har-Peled, S., and Varadarajan, K.~R.
\newblock Approximating extent measures of points.
\newblock \emph{Journal of the ACM (JACM)}, 51\penalty0 (4):\penalty0 606--635,
  2004.

\bibitem[Allen-Zhu(2017)]{allen2017katyusha}
Allen-Zhu, Z.
\newblock Katyusha: The first direct acceleration of stochastic gradient
  methods.
\newblock \emph{The Journal of Machine Learning Research}, 18\penalty0
  (1):\penalty0 8194--8244, 2017.

\bibitem[Allen-Zhu et~al.(2016)Allen-Zhu, Yuan, and
  Sridharan]{allen2016exploiting}
Allen-Zhu, Z., Yuan, Y., and Sridharan, K.
\newblock Exploiting the structure: Stochastic gradient methods using raw
  clusters.
\newblock In \emph{Advances in Neural Information Processing Systems}, pp.\
  1642--1650, 2016.

\bibitem[Asi \& Duchi(2019)Asi and Duchi]{asi2019importance}
Asi, H. and Duchi, J.~C.
\newblock The importance of better models in stochastic optimization.
\newblock \emph{arXiv preprint arXiv:1903.08619}, 2019.

\bibitem[Ba et~al.(2016)Ba, Kiros, and Hinton]{ba2016layer}
Ba, J.~L., Kiros, J.~R., and Hinton, G.~E.
\newblock Layer normalization.
\newblock \emph{arXiv preprint arXiv:1607.06450}, 2016.

\bibitem[Bertsekas(1996)]{bertsekas1996incremental}
Bertsekas, D.~P.
\newblock Incremental least squares methods and the extended kalman filter.
\newblock \emph{SIAM Journal on Optimization}, 6\penalty0 (3):\penalty0
  807--822, 1996.

\bibitem[Bertsekas(2015{\natexlab{a}})]{bertsekas2015convex}
Bertsekas, D.~P.
\newblock \emph{Convex optimization algorithms}.
\newblock Athena Scientific Belmont, 2015{\natexlab{a}}.

\bibitem[Bertsekas(2015{\natexlab{b}})]{bertsekas2015incremental}
Bertsekas, D.~P.
\newblock Incremental gradient, subgradient, and proximal methods for convex
  optimization: A survey.
\newblock \emph{arXiv preprint arXiv:1507.01030}, 2015{\natexlab{b}}.

\bibitem[Campbell \& Broderick(2018)Campbell and Broderick]{Campbell18_ICML}
Campbell, T. and Broderick, T.
\newblock Bayesian coreset construction via greedy iterative geodesic ascent.
\newblock In \emph{International Conference on Machine Learning}, 2018.

\bibitem[Chung et~al.(1954)]{chung1954stochastic}
Chung, K.~L. et~al.
\newblock On a stochastic approximation method.
\newblock \emph{The Annals of Mathematical Statistics}, 25\penalty0
  (3):\penalty0 463--483, 1954.

\bibitem[Cohen et~al.(2017)Cohen, Musco, and Musco]{cohen2017input}
Cohen, M.~B., Musco, C., and Musco, C.
\newblock Input sparsity time low-rank approximation via ridge leverage score
  sampling.
\newblock In \emph{Proceedings of the Twenty-Eighth Annual ACM-SIAM Symposium
  on Discrete Algorithms}, pp.\  1758--1777. SIAM, 2017.

\bibitem[Defazio \& Bottou(2019)Defazio and Bottou]{defazio2019ineffectiveness}
Defazio, A. and Bottou, L.
\newblock On the ineffectiveness of variance reduced optimization for deep
  learning.
\newblock In \emph{Advances in Neural Information Processing Systems}, pp.\
  1755--1765, 2019.

\bibitem[Defazio et~al.(2014)Defazio, Bach, and
  Lacoste-Julien]{defazio2014saga}
Defazio, A., Bach, F., and Lacoste-Julien, S.
\newblock Saga: A fast incremental gradient method with support for
  non-strongly convex composite objectives.
\newblock In \emph{Advances in neural information processing systems}, pp.\
  1646--1654, 2014.

\bibitem[Duchi et~al.(2011)Duchi, Hazan, and Singer]{duchi2011adaptive}
Duchi, J., Hazan, E., and Singer, Y.
\newblock Adaptive subgradient methods for online learning and stochastic
  optimization.
\newblock \emph{Journal of Machine Learning Research}, 12\penalty0
  (Jul):\penalty0 2121--2159, 2011.

\bibitem[Frostig et~al.(2015)Frostig, Ge, Kakade, and
  Sidford]{frostig2015regularizing}
Frostig, R., Ge, R., Kakade, S., and Sidford, A.
\newblock Un-regularizing: approximate proximal point and faster stochastic
  algorithms for empirical risk minimization.
\newblock In \emph{International Conference on Machine Learning}, pp.\
  2540--2548, 2015.

\bibitem[Glorot \& Bengio(2010)Glorot and Bengio]{glorot2010understanding}
Glorot, X. and Bengio, Y.
\newblock Understanding the difficulty of training deep feedforward neural
  networks.
\newblock In \emph{Proceedings of the thirteenth international conference on
  artificial intelligence and statistics}, pp.\  249--256, 2010.

\bibitem[G{\"u}rb{\"u}zbalaban et~al.(2015)G{\"u}rb{\"u}zbalaban, Ozdaglar, and
  Parrilo]{gurbuzbalaban2015random}
G{\"u}rb{\"u}zbalaban, M., Ozdaglar, A., and Parrilo, P.
\newblock Why random reshuffling beats stochastic gradient descent.
\newblock \emph{arXiv preprint arXiv:1510.08560}, 2015.

\bibitem[Gurbuzbalaban et~al.(2017)Gurbuzbalaban, Ozdaglar, and
  Parrilo]{gurbuzbalaban2017convergence}
Gurbuzbalaban, M., Ozdaglar, A., and Parrilo, P.~A.
\newblock On the convergence rate of incremental aggregated gradient
  algorithms.
\newblock \emph{SIAM Journal on Optimization}, 27\penalty0 (2):\penalty0
  1035--1048, 2017.

\bibitem[Har-Peled \& Mazumdar(2004)Har-Peled and Mazumdar]{har2004coresets}
Har-Peled, S. and Mazumdar, S.
\newblock On coresets for k-means and k-median clustering.
\newblock In \emph{Proceedings of the thirty-sixth annual ACM symposium on
  Theory of computing}, pp.\  291--300. ACM, 2004.

\bibitem[Hofmann et~al.(2015)Hofmann, Lucchi, Lacoste-Julien, and
  McWilliams]{hofmann2015variance}
Hofmann, T., Lucchi, A., Lacoste-Julien, S., and McWilliams, B.
\newblock Variance reduced stochastic gradient descent with neighbors.
\newblock In \emph{Advances in Neural Information Processing Systems}, pp.\
  2305--2313, 2015.

\bibitem[Ioffe \& Szegedy(2015)Ioffe and Szegedy]{ioffe2015batch}
Ioffe, S. and Szegedy, C.
\newblock Batch normalization: Accelerating deep network training by reducing
  internal covariate shift.
\newblock In \emph{International Conference on Machine Learning}, pp.\
  448--456, 2015.

\bibitem[Johnson \& Zhang(2013)Johnson and Zhang]{johnson2013accelerating}
Johnson, R. and Zhang, T.
\newblock Accelerating stochastic gradient descent using predictive variance
  reduction.
\newblock In \emph{Advances in neural information processing systems}, pp.\
  315--323, 2013.

\bibitem[Katharopoulos \& Fleuret(2018)Katharopoulos and
  Fleuret]{katharopoulos2018not}
Katharopoulos, A. and Fleuret, F.
\newblock Not all samples are created equal: Deep learning with importance
  sampling.
\newblock In \emph{International Conference on Machine Learning}, pp.\
  2525--2534, 2018.

\bibitem[Kaufman et~al.(1987)Kaufman, Rousseeuw, and
  Dodge]{kaufman1987clustering}
Kaufman, L., Rousseeuw, P., and Dodge, Y.
\newblock Clustering by means of medoids in statistical data analysis based on
  the, 1987.

\bibitem[Kingma \& Ba(2014)Kingma and Ba]{kingma2014adam}
Kingma, D.~P. and Ba, J.
\newblock Adam: A method for stochastic optimization.
\newblock \emph{arXiv preprint arXiv:1412.6980}, 2014.

\bibitem[Li et~al.(2013)Li, Miller, and Peng]{li2013iterative}
Li, M., Miller, G.~L., and Peng, R.
\newblock Iterative row sampling.
\newblock In \emph{2013 IEEE 54th Annual Symposium on Foundations of Computer
  Science}, pp.\  127--136. IEEE, 2013.

\bibitem[Lin \& Bilmes(2012)Lin and Bilmes]{lin2012learning}
Lin, H. and Bilmes, J.~A.
\newblock Learning mixtures of submodular shells with application to document
  summarization.
\newblock \emph{arXiv preprint arXiv:1210.4871}, 2012.

\bibitem[Lin et~al.(2009)Lin, Bilmes, and Xie]{lin2009-submodsum}
Lin, H., Bilmes, J., and Xie, S.
\newblock Graph-based submodular selection for extractive summarization.
\newblock In \emph{Proc.\ IEEE Automatic Speech Recognition and Understanding
  (ASRU)}, Merano, Italy, December 2009.

\bibitem[Lin et~al.(2015)Lin, Mairal, and Harchaoui]{lin2015universal}
Lin, H., Mairal, J., and Harchaoui, Z.
\newblock A universal catalyst for first-order optimization.
\newblock In \emph{Advances in Neural Information Processing Systems}, pp.\
  3384--3392, 2015.

\bibitem[Lucic et~al.(2017)Lucic, Faulkner, Krause, and
  Feldman]{lucic2017training}
Lucic, M., Faulkner, M., Krause, A., and Feldman, D.
\newblock Training gaussian mixture models at scale via coresets.
\newblock \emph{The Journal of Machine Learning Research}, 18\penalty0
  (1):\penalty0 5885--5909, 2017.

\bibitem[Mangasariany \& Solodovy(1994)Mangasariany and
  Solodovy]{mangasariany1994serial}
Mangasariany, O. and Solodovy, M.
\newblock Serial and parallel backpropagation convergence via nonmonotone
  perturbed minimization.
\newblock 1994.

\bibitem[Minoux(1978)]{minoux1978accelerated}
Minoux, M.
\newblock Accelerated greedy algorithms for maximizing submodular set
  functions.
\newblock In \emph{Optimization techniques}, pp.\  234--243. Springer, 1978.

\bibitem[Mirzasoleiman et~al.(2015{\natexlab{a}})Mirzasoleiman, Badanidiyuru,
  Karbasi, Vondr{\'a}k, and Krause]{mirzasoleiman2015lazier}
Mirzasoleiman, B., Badanidiyuru, A., Karbasi, A., Vondr{\'a}k, J., and Krause,
  A.
\newblock Lazier than lazy greedy.
\newblock In \emph{Twenty-Ninth AAAI Conference on Artificial Intelligence},
  2015{\natexlab{a}}.

\bibitem[Mirzasoleiman et~al.(2015{\natexlab{b}})Mirzasoleiman, Karbasi,
  Badanidiyuru, and Krause]{mirzasoleiman2015distributed}
Mirzasoleiman, B., Karbasi, A., Badanidiyuru, A., and Krause, A.
\newblock Distributed submodular cover: Succinctly summarizing massive data.
\newblock In \emph{Advances in Neural Information Processing Systems}, pp.\
  2881--2889, 2015{\natexlab{b}}.

\bibitem[Mirzasoleiman et~al.(2016)Mirzasoleiman, Zadimoghaddam, and
  Karbasi]{mirzasoleiman2016fast}
Mirzasoleiman, B., Zadimoghaddam, M., and Karbasi, A.
\newblock Fast distributed submodular cover: Public-private data summarization.
\newblock In \emph{Advances in Neural Information Processing Systems}, pp.\
  3594--3602, 2016.

\bibitem[Musco \& Musco(2017)Musco and Musco]{musco2017recursive}
Musco, C. and Musco, C.
\newblock Recursive sampling for the nystrom method.
\newblock In \emph{Advances in Neural Information Processing Systems}, pp.\
  3833--3845, 2017.

\bibitem[Nedi{\'c} \& Bertsekas(2001)Nedi{\'c} and
  Bertsekas]{nedic2001convergence}
Nedi{\'c}, A. and Bertsekas, D.
\newblock Convergence rate of incremental subgradient algorithms.
\newblock In \emph{Stochastic optimization: algorithms and applications}, pp.\
  223--264. Springer, 2001.

\bibitem[Nemhauser et~al.(1978)Nemhauser, Wolsey, and Fisher]{nemhauser1978}
Nemhauser, G., Wolsey, L., and Fisher, M.
\newblock An analysis of approximations for maximizing submodular set
  functions---i.
\newblock \emph{Mathematical Programming}, 14\penalty0 (1):\penalty0 265--294,
  1978.

\bibitem[Qian(1999)]{qian1999momentum}
Qian, N.
\newblock On the momentum term in gradient descent learning algorithms.
\newblock \emph{Neural networks}, 12\penalty0 (1):\penalty0 145--151, 1999.

\bibitem[Roux et~al.(2012)Roux, Schmidt, and Bach]{roux2012stochastic}
Roux, N.~L., Schmidt, M., and Bach, F.~R.
\newblock A stochastic gradient method with an exponential convergence \_rate
  for finite training sets.
\newblock In \emph{Advances in neural information processing systems}, pp.\
  2663--2671, 2012.

\bibitem[Shalev-Shwartz \& Zhang(2013)Shalev-Shwartz and
  Zhang]{shalev2013stochastic}
Shalev-Shwartz, S. and Zhang, T.
\newblock Stochastic dual coordinate ascent methods for regularized loss
  minimization.
\newblock \emph{Journal of Machine Learning Research}, 14\penalty0
  (Feb):\penalty0 567--599, 2013.

\bibitem[Solodov(1998)]{solodov1998incremental}
Solodov, M.~V.
\newblock Incremental gradient algorithms with stepsizes bounded away from
  zero.
\newblock \emph{Computational Optimization and Applications}, 11\penalty0
  (1):\penalty0 23--35, 1998.

\bibitem[Strubell et~al.(2019)Strubell, Ganesh, and
  McCallum]{strubell2019energy}
Strubell, E., Ganesh, A., and McCallum, A.
\newblock Energy and policy considerations for deep learning in nlp.
\newblock \emph{arXiv preprint arXiv:1906.02243}, 2019.

\bibitem[Tseng(1998)]{tseng1998incremental}
Tseng, P.
\newblock An incremental gradient (-projection) method with momentum term and
  adaptive stepsize rule.
\newblock \emph{SIAM Journal on Optimization}, 8\penalty0 (2):\penalty0
  506--531, 1998.

\bibitem[Wei et~al.(2015)Wei, Iyer, and Bilmes]{wei2015submodularity}
Wei, K., Iyer, R., and Bilmes, J.
\newblock Submodularity in data subset selection and active learning.
\newblock In \emph{International Conference on Machine Learning}, pp.\
  1954--1963, 2015.

\bibitem[Wolsey(1982)]{wolsey1982analysis}
Wolsey, L.~A.
\newblock An analysis of the greedy algorithm for the submodular set covering
  problem.
\newblock \emph{Combinatorica}, 2\penalty0 (4):\penalty0 385--393, 1982.

\bibitem[Xiao \& Zhang(2014)Xiao and Zhang]{xiao2014proximal}
Xiao, L. and Zhang, T.
\newblock A proximal stochastic gradient method with progressive variance
  reduction.
\newblock \emph{SIAM Journal on Optimization}, 24\penalty0 (4):\penalty0
  2057--2075, 2014.

\bibitem[Zeiler(2012)]{zeiler2012adadelta}
Zeiler, M.~D.
\newblock Adadelta: an adaptive learning rate method.
\newblock \emph{arXiv preprint arXiv:1212.5701}, 2012.

\bibitem[Zhi-Quan \& Paul(1994)Zhi-Quan and Paul]{zhi1994analysis}
Zhi-Quan, L. and Paul, T.
\newblock Analysis of an approximate gradient projection method with
  applications to the backpropagation algorithm.
\newblock \emph{Optimization Methods and Software}, 4\penalty0 (2):\penalty0
  85--101, 1994.

\end{thebibliography}
\bibliographystyle{icml2020}

\appendix
\onecolumn

\section{Convergence Rate Analysis}
We firs proof the following Lemma which is an extension of the [\cite{chung1954stochastic}, Lemma 4].
\begin{lemma}\label{lemma:seq}
	Let $u_k \geq 0$ be a sequence of real numbers. Assume there exist $k_0$ such that $$u_{k+1} \leq (1-\frac{c}{k}) u_k + \frac{e}{k^{p}} + \frac{d}{k^{p+1}}, \quad \forall k \geq k_0,$$
	where $e > 0, d > 0, c > 0$ are given real numbers.
	Then
	\begin{align}
	u_k &\leq (dk^{-1}+e)(c-p+1)^{-1} k^{-p+1} + o(k^{-p+1} )  &\text{for}~ c>p-1, p \geq 1 \label{eq:cnv1}\\
	u_k &= O(k^{-c} \log k)  &\text{for}~ c=p-1, p>1\\
	u_k &= O(k^{-c}) 		   &\text{for}~ c<p-1, p>1\\
	\end{align}
\end{lemma}

\begin{proof}
	Let $c > p-1$ and $v_k = k^{p-1} u_k - \frac{d}{k(c-p+1)} - \frac{e}{c-p+1}$. Then, using Taylor approximation $(1+\frac{1}{k})^p=(1+\frac{p}{k})+o(\frac{1}{k})$ we can write
	\begin{align}
	\hspace{-.6cm}
	v_{k+1} 
	&= (k+1)^{p-1} u_{k+1} - \frac{d}{(k+1)(c-p+1)} -  \frac{e}{c-p+1}\\
	& \leq k^{p-1} (1+\frac{1}{k})^{p-1} \Big( (1-\frac{c}{k}) u_k + \frac{e}{k^p} + \frac{d}{k^{p+1}} \Big) - \frac{d}{(k+1)(c-p+1)} -  \frac{e}{c-p+1}\\
	&= k^{p-1} u_k \Big( 1- \frac{c-p+1}{k} + o(\frac{1}{k}) \Big) 
	+ \frac{e}{k}\Big(1+\frac{p-1}{k} + o(\frac{1}{k})\Big)\\
	& ~~~~ + \frac{d}{k^2} \Big(1 + \frac{p-1}{k} + o(\frac{1}{k}) \Big) - \frac{d}{(k+1)(c-p+1)} -  \frac{e}{c-p+1}\\
	&= \Big(v_k + \frac{d}{k(c-p+1)} + \frac{e}{c-p+1}\Big) \Big(1-\frac{c-p+1}{k} + o(\frac{1}{k})\Big)  \\
	&~~~~ + \frac{e}{k}\Big(1+\frac{p-1}{k} + o(\frac{1}{k})\Big) + \frac{d}{k^2} \Big(1 + \frac{p-1}{k} + o(\frac{1}{k}) \Big)\\
	&~~~~ - \frac{d}{(k+1)(c-p+1)} -  \frac{e}{c-p+1}\\
	&= v_k\Big(1-\frac{c-p+1}{k} + o(\frac{1}{k})\Big) + \frac{d/(c-p+1)}{k(k+1)} + \frac{e(p-1)}{k^2} + \frac{d(p-1)}{k^3}+o(\frac{1}{k^2})
	\end{align}
	Note that for $v_k$, we have 
	$$\sum_{k=0}^\infty \Big(1-\frac{c-p+1}{k} + o(\frac{1}{k})\Big) = \infty$$ and $$\Big(\frac{d/(c-p+1)}{k(k+1)} + \frac{e(p-1)}{k^2} + \frac{d(p-1)}{k^3}+o(\frac{1}{k^2})\Big) \Big(1-\frac{c-p+1}{k} + o(\frac{1}{k})\Big)^{-1} \rightarrow 0.$$ Therefore,
	$\lim_{k \rightarrow \infty} v_k \leq 0$, and we get Eq. \ref{eq:cnv1}.
	For $p=1$, we have $u_k \leq \frac{e}{c}$. Hence, $u_k$ converges  into the region $u \leq  \frac{e}{c}$, with ratio $1-\frac{c}{k}$.
	
	Moreover, for $p-1 \geq c$ we have
	\begin{align}
	v_{k+1} &= u_{k+1} (k+1)^c \leq \Big[(1-\frac{c}{k})u_k+\frac{e}{k^p}+\frac{d}{k^{p+1}}\Big] k^c \Big( 1+\frac{c}{k} + \frac{c^2}{2k^2}+o(\frac{1}{k^2}) \Big)\\
	& =\Big( 1-\frac{c^2}{2k^2}+o(\frac{1}{k^2}) \Big)v_k + \frac{d}{k^{p-c+1}} \Big(1+ O(\frac{1}{k})\Big) + \frac{e}{k^{p-c}} \Big(1+\frac{c}{k} +O(\frac{1}{k^2})\Big)\\
	& \leq v_k 
	+ \frac{e'}{k^{p-c}} 
	\end{align}
	for sufficiently large $k$. Summing over $k$, we obtain that $v_k$ is bounded for $p-1>c$ (since the series $\sum_{k=1}^\infty (1/k^{\alpha})$ converges for $\alpha>1$) and $v_k=O(\log k)$ for $p=c+1$ (since $\sum_{i=1}^k (1/i) = O(\log k)$). 
\end{proof}
In addition, based on [\cite{chung1954stochastic}, Lemma 5] for $u_k \geq 0$, we can write
\begin{equation}
u_{k+1}\leq (1-\frac{c}{k^s})u_k + \frac{e}{k^p} + \frac{d}{k^t}, \quad\quad 0<s<1, s\leq p < t.
\end{equation}
Then, we have
\begin{equation}
u_k \leq \frac{e}{c} \frac{1}{k^{p-s}} + o(\frac{1}{k^{p-s}}).
\end{equation}

\subsection{Convergence Rate for Strongly Convex Functions}\label{app:thm1}
\hide{
\begin{lemma}\label{clm} 
	Assume $f_i$s are convex and there is a subset $S$ of size $r$ with corresponding per-element stepsizes $\{\gamma\}_j$ that estimates the full gradient by an error of at most $\epsilon$, i.e., $ \| \sum_{j \in S} \gamma_{j} \nabla f_j(w) - \sum_{i \in V} \nabla f_i(w) \| \leq \epsilon $. Then we have
\begin{align*}
    \sum_{j \in S} (f_j(w_{k})
    - f_j(\opt)) \geq
    \mu \| w_k - \opt \|^2 - \big(\epsilon + \sum_{i\in V} \nabla f_i(w_{k}) \big) \cdot  (w_k - \opt)
\end{align*}
\end{lemma}
\begin{proof}
From the convexity of each component we now that
\begin{equation}\label{eq:convf}
	\sum_{i \in V} (f_i(w_{k}) - f_i(\opt)) \leq \sum_{i\in V} \nabla f_i(w_{k}) \cdot  (w_k - \opt)
\end{equation}
Moreover, if $\sum_{j \in S} f_j(w_{k})$ is strongly convex, then we have
\begin{align}\label{eq:strf}
\sum_{j\in S} (f_j(w_k) - f_j(\opt)) 
&\geq \sum_{j\in S} \gamma_j \nabla f_j(\opt)(w_k-\opt) + \frac{\mu}{2} \| w_k - \opt \|^2
\end{align}

---->
Subtracting Eq. \ref{eq:convf} from Eq. \ref{eq:strf} and taking the norm of both sides, we get
\begin{align}
    |\sum_{j \in S} (f_j&(w_{k})
    - f_j(\opt)) - \sum_{i\in V} (f_i(w_k) - f_i(\opt))|  \nonumber\\
    &\geq 
    \| \sum_{j\in S} \gamma_j \nabla f_j(\opt)\cdot (w_k-\opt) + \frac{\mu}{2} \| w_k - \opt \|^2 
    - \sum_{i\in V} \nabla f_i(w_{k}) \cdot  (w_k - \opt)\| \\
    &= \| (\sum_{j\in S} \gamma_j \nabla f_j(\opt) - \sum_{i\in V} \nabla f_i(w_{k})) \cdot  (w_k - \opt) \|+ \frac{\mu}{2} \| w_k - \opt \|^2  \\
    &=
    \| \Big(\sum_{j\in S} \gamma_j \nabla f_j(\opt) - \sum_{i\in V} \nabla f_i(\opt)  -\big(\sum_{i\in V} \nabla f_i(w_k) - \sum_{i\in V} \nabla f_i(\opt)\big)\Big) \cdot (w_k-\opt)\| \nonumber\\
    & \qquad + \frac{\mu}{2} \| w_k - \opt \|^2 \\
    &\geq
    \| \big(\sum_{j\in S} \gamma_j \nabla f_j(\opt) - \sum_{i\in V} \nabla f_i(\opt)\big) \cdot (w_k-\opt)\| \nonumber\\ 
    & \qquad - \| \big(\sum_{i\in V} \nabla f_i(w_k) - \sum_{i\in V} \nabla f_i(\opt)\big) \cdot (w_k-\opt)\| + \frac{\mu}{2} \| w_k - \opt \|^2 \\
    &\geq
    \| \big(\sum_{j\in S} \gamma_j \nabla f_j(\opt) - \sum_{i\in V} \nabla f_i(\opt)\big) \cdot (w_k-\opt)\| \nonumber\\ 
    & \qquad - \| \big(\sum_{i\in V} \nabla f_i(w_k) - \sum_{i\in V} \nabla f_i(\opt)\big) \cdot (w_k-\opt)\| + \frac{\mu}{2} \| w_k - \opt \|^2
\end{align}

Moreover, if $\sum_{j \in S} f_j(w_{k})$ is strongly convex, then we have
\begin{align}
\sum_{j\in S} (f_j(w_k) - f_j(\opt)) 
&\geq \sum_{j\in S} \gamma_j \nabla f_j(\opt)(w_k-\opt) + \frac{\mu}{2} \| w_k - \opt \|^2 \\
&\geq (\sum_{i\in V}\nabla f_i(\opt)-\epsilon)(w_k-\opt) + \frac{\mu}{2} \| w_k - \opt \|^2 \\
&\geq -\epsilon \cdot (w_k-\opt) + \frac{\mu}{2} \| w_k - \opt \|^2.
\end{align}
Subtracting the above inequalities, we get
\begin{align}
    \sum_{j \in S} (f_j(w_{k})
    - f_j(\opt)) - \sum_{i\in V} (f_i(w_k) - f_i(\opt)) 
    \geq 
    \frac{\mu}{2} \| &w_k - \opt \|^2 -\epsilon \cdot (w_k-\opt) \nonumber \\
    &- \sum_{i\in V} \nabla f_i(w_{k}) \cdot  (w_k - \opt). 
\end{align}
Using strong convexity of $f$ and that $f(\opt)$ is optimum we have
\begin{align}
    \sum_{j \in S}  
    (f_j(w_{k}) - f_j(\opt)) 
    &\geq \sum_{i\in V} (f_i(w_k) - f_i(\opt)) +
    \frac{\mu}{2} \| w_k - \opt \|^2 \nonumber\\
    &\hspace{1cm} - \big(\epsilon + \sum_{i\in V} \nabla f_i(w_{k}) \big) \cdot  (w_k - \opt)\\
    &\geq \mu \| w_k - \opt \|^2 - \big(\epsilon + \sum_{i\in V} \nabla f_i(w_{k}) \big) \cdot  (w_k - \opt).
\end{align}
=====> before
From the convexity of each component we get the following inequalities
	\begin{eqnarray}
	\sum_{j \in S} (f_j(w_{k}) - f_j(\opt)) \leq \| \sum_{j\in S} \gamma_{j} \nabla f_j(w_{k}) \| \cdot \| w_k - \opt \|,\\
	\sum_{i\in V} (f_i(w_{k}) - f_i(\opt)) \leq \| \sum_{i\in V} \nabla f_i(w_{k}) \| \cdot \| w_k - \opt \|
	\end{eqnarray}
	By subtracting the above inequalities we get
	\begin{eqnarray}
	 \mid{\sum_{j \in S} \gamma_{j} (f_j(w_{k}) - f_j(\opt)) - \sum_{i\in V} (f_i(w_{k}) - f_i(\opt))}\mid 
	 &&\hspace{5cm} \nonumber\\
	&& \hspace{-4cm} \leq ( \| \sum_{j \in S} \gamma_{j} \nabla f_j(w_{k}) \| - \| \sum_{i\in V} \nabla f_i(w_{k}) \| ) \cdot \| w_k - \opt \|  \label{eq:app1}\\
	&&\hspace{-4cm}   \leq \| \sum_{j \in S} \gamma_{j} \nabla f_j(w_{k}) - \sum_{i \in V} \nabla f_i(w_{k}) \|  \cdot \| w_k - \opt \| \\
	&&\hspace{-4cm}  \leq \epsilon \| w_k - \opt \|
\end{eqnarray}

\end{proof}
}
\subsection*{Proof of Theorem \ref{thm:strong}}

We now provide the convergence rate for strongly convex functions building on the analysis of \cite{nedic2001convergence}. For non-smooth functions, gradients can be replaced by sub-gradients. 

Let $w_k = w_0^k$. For every IG update on subset $S$ we have
\begin{align}
\| w_{j}^k - \opt\|^2 
& =  \| w_{j-1}^k - \alpha_{k} \gamma_j \nabla f_{j}(w_{j-1}^k) - \opt\|^2 \\
& =  \| w_{j-1}^k - \opt\|^2 - 2\alpha_{k} \gamma_j \nabla f_j(w_{j-1}^k)(w_{j-1}^k - \opt) + \alpha_{k}^2 \| \gamma_j \nabla f_j(w_{j-1}^k)\|^2\\
& \leq  \| w_{j-1}^k - \opt\|^2 - 2\alpha_{k}  (f_j(w_{j-1}^k) - f_i(\opt)) + \alpha_{k}^2 \|\gamma_i \nabla f_j(w_{j-1}^k)\|^2.
\end{align}
Adding the above inequalities over elements of $S$ we get 
\begin{align}
\| w_{k+1} - \opt\|^2 
\leq \| w_{k} - \opt\|^2 &- 2\alpha_{k} \sum_{j\in S}  (f_i(w_{j-1}^k) - f_j(\opt)) +  \alpha_{k}^2 \sum_{j\in S} \| \gamma_j \nabla f_j(w_{j-1}^k)\|^2 \\
= \| w_{k} - \opt\|^2 &- 2\alpha_k \sum_{j\in S}  (f_j(w_k) - f_i(\opt)) \nonumber \\
& + 2 \alpha_k \sum_{j\in S}  (f_j(w_{j-1}^k) - f_j(w_k)) + \alpha_k^2 \sum_{j\in S} \|\gamma_j\nabla f_j(w_{j-1}^k)\|^2 
\end{align}
%
Using strong convexity 
we can write
\begin{align}
\| w_{k+1} - \opt\|^2 \leq \| w_{k} - \opt\|^2 
&- 2\alpha_k \big(\sum_{j \in S} \gamma_j \nabla f_j(\opt)\cdot (w_k-\opt) + \frac{\mu}{2}\|w_k-\opt\|^2 \big) \nonumber\\ 
& + 2 \alpha_{k}\sum_{j\in S} (f_j(w_{j-1}^k) - f_j(w_k)) + \alpha_{k}^2 \sum_{j \in S} \| \gamma_j \nabla f_j(w_{j-1}^k)\|^2 
\end{align}
\hide{
\begin{eqnarray}
\| w_{k+1} - \opt\|^2 
&\leq& \| w_{k} - \opt\|^2 - 2\alpha_k(f(w_k) - f_* - \epsilon \| w_k - \opt \| ) + \\ 
&& 2 \alpha_{k}\sum_{i\in S} \gamma_i(f_i(w_{i-1}^k) - f_i(w_k)) + \alpha_{k}^2 \sum_{i \in S} \| \gamma_i \nabla f_i(w_{i-1}^k)\|^2 
\end{eqnarray}
}
Using Cauchy–Schwarz inequality, we know
\begin{align}
    | \sum_{j \in S} \gamma_j \nabla f_j(\opt)\cdot (w_k -\opt) | \leq \| \sum_{j \in S} \gamma_j \nabla f_j(\opt) \| \cdot \| w_k - \opt \|.
\end{align}
Hence,
\begin{align}\label{eq:cs}
    -\sum_{j \in S} \gamma_j \nabla f_j(\opt)\cdot (w_k -\opt) \leq \| \sum_{j \in S} \gamma_j \nabla f_j(\opt) \| \cdot \| w_k - \opt \|.  
\end{align}
From reverse triangle inequality, and the facts that $S$ is chosen in a way that $\|\sum_{i \in V}\nabla f_i(\opt) - \sum_{j\in S}\gamma_j\nabla f_j(\opt)\| \leq \epsilon$, and that $\sum_{i\in V} \nabla f_i(\opt)=0$ 
we have $\| \sum_{j\in S}\gamma_j\nabla f_j(\opt)\| \leq \|\sum_{i \in V}\nabla f_i(\opt) \| + \epsilon = \epsilon$. Therefore 
\begin{align}
    \| \sum_{j \in S} \gamma_j \nabla f_j(\opt) \| \cdot \| w_k - \opt \| \leq \epsilon \cdot \|w_k-\opt\| 
\end{align}

For a continuously differentiable function, the following condition is implied by strong convexity condition
\begin{equation}\label{eq:xerror}
\| w_k - \opt \| \leq \frac{1}{\mu} \| \sum_{j\in S} \gamma_j \nabla f_j(w_k) \| .
\end{equation}
Assuming gradients have a bounded norm $\max_{\substack{x \in \mathcal{X},\\j \in V}} \|\nabla f_j(w)\| \leq C$, and the fact that $\sum_{j\in S}\gamma_j = n$ we can write
\begin{equation}
    \| \sum_{j\in S} \gamma_j \nabla f_j(w_k) \| \leq n\cdot C .
\end{equation}
Thus for initial distance $\| w_0 - \opt \|=d_0$, we have
\begin{equation}\label{eq:R}
    \| w_k - \opt \| \leq \min (n\cdot C, d_0) = R
\end{equation}
Putting Eq. \ref{eq:cs} to Eq. \ref{eq:R} together we get
\begin{align}
\| w_{k+1} - \opt\|^2 
\leq (1-\alpha_k \mu)
&\| w_{k} - \opt\|^2  
 + 2\alpha_k \epsilon R/\mu 
\nonumber\\
& + 2 \alpha_{k}\sum_{j\in S} (f_j(w_{j-1,k}) - f_j(w_k)) + \alpha_{k}^2 r\gamma_{\max}^2 C^2.
\end{align}
Now, from convexity of every $f_j$ for $j \in S$ we have that 
\begin{equation}
    f_j(w_k)-f_j(w_{j-1}^k) \leq \| \gamma_j \nabla f_j(w_k) \| \cdot \| w_{j-1}^k - w_k \|.
\end{equation}

In addition, incremental updates gives us
\begin{equation}
    \| w_{j-1}^k - w_k \| \leq \alpha_{k} \sum_{i=1}^{j-1} \|  \gamma_i \nabla f_{i}(w_{i-1}^k) \| \leq \alpha_k (j-1) \gamma_{\max} C.
\end{equation}
Therefore, we get
\begin{align}
2\alpha_k \sum_{j\in S} (f_j(w_k)-f_j(w_{j-1}^k)) &+ \alpha_k^2 r\gamma_{\max}^2 C^2 \nonumber\\
&\leq 2 \alpha_k \sum_{i=1}^r \gamma_{\max} C \cdot \alpha_k (j-1) \gamma_{\max} C + \alpha_k^2 r\gamma_{\max}^2 C^2 \\
&=  \alpha_k^2 r^2 \gamma_{\max}^2 C^2
\end{align}
Hence,
\begin{equation}\label{eq:str_rate}
\| w_{k+1} - \opt\|^2 
\leq (1-\alpha_k \mu) \| w_{k} - \opt\|^2  + 2\alpha_k \epsilon R/\mu + \alpha_k^2 r^2 \gamma_{\max}^2 C^2.
\end{equation}
where $\gamma_{\text{max}}$ is the size of the largest cluster, and $C$ is the upperbound on the gradients.

For $0<\tau\leq 1$, the theorem follows by applying Lemma \ref{lemma:seq} to Eq. \ref{eq:str_rate}, with $c=\alpha \mu$, $e=2\alpha\epsilon R/\mu$, and $d=\alpha^2 r^2 \gamma_{\max}^2 C^2$.

For $\tau=0$, where we have a constant step size $\alpha_k = \alpha \leq \frac{1}{\mu}$, we get
\begin{align}
\| w_{k+1} - \opt\|^2 
\leq (1-&\alpha \mu)^{k+1} \| w_{0} - \opt\|^2  \nonumber\\
&+ 2 \alpha \epsilon R \sum_{j=0}^{k} (1-\alpha \mu)^{j} /\mu 
+ \alpha^2 r^2 \gamma_{\max}^2 C^2 \sum_{j=0}^{k} (1-\alpha \mu)^{j}
\end{align}
Since $\sum_{j=0}^{k} (1-\alpha \mu)^{j} \leq \frac{1}{\alpha \mu}$, we get
\begin{equation}\label{eq:rate_strong}
\| w_{k+1} - \opt\|^2 
\leq (1-\alpha \mu)^{k+1} \| w_{0} - \opt\|^2  + 2 \alpha \epsilon R/(\alpha\mu^2) + \alpha^2 r^2 \gamma_{\max}^2 C^2/(\alpha\mu),
\end{equation}
and therefore,
\begin{equation}
\| w_{k+1} - \opt\|^2 
\leq (1-\alpha \mu)^{k+1} \| w_{k} - \opt\|^2  + 2 \epsilon R /\mu^2 + \alpha r^2 \gamma_{\max}^2 C^2 / \mu.
\end{equation}

\subsection{Convergence Rate for Strongly Convex and Smooth Component Functions}\label{app:thm2}
\subsection*{Proof of Theorem \ref{thm:smooth}}
IG updates for cycle $k$ on subset $S$ can be written as
\begin{align}
w_{k+1} = w_k - \alpha_k (\sum_{j\in S} \gamma_j  \nabla f_j(w_k) - e_k) \label{eq:cmp}\\ 
e_k = \sum_{j\in S} \gamma_i(\nabla f_j(w_k) - \nabla f_j(w_{j-1}^k)) 
\end{align}
Building on the analysis of \cite{gurbuzbalaban2015random}, 
for convex and twice continuously differentiable function, we can write
\begin{eqnarray}\label{eq:ak}
\sum_{j\in S} \gamma_j \nabla f_j(w_k) - \sum_{j\in S} \gamma_j \nabla f_j(\opt) = A_k^r(w_k - \opt) 
\end{eqnarray}
where $A_k^r = \int_0^1 \nabla^2 f(\opt + \tau (w_k - \opt)) d\tau$  is average of the Hessian matrices corresponding to the $r$ (weighted) elements of $S$ on the interval $[w_k, \opt]$.

From Eq. \ref{eq:ak} 
we have
\begin{equation}
\sum_{i\in V} (\nabla f_i(w_k)-\nabla f_i(\opt)) - \sum_{j\in S} \gamma_j (\nabla f_j(w_k)-\nabla f_j(\opt))= A_k(w_k - \opt)  - A_k^r(w_k - \opt), 
\end{equation}
where $A_k$ is average of the Hessian matrices corresponding to all the $n$ component functions on the interval $[w_k, \opt]$. Taking norm of both sides and noting that $\sum_{i \in V} f_i(\opt)=0$ and hence $\|\sum_{j \in S} \gamma_j f_j(\opt)\|\leq \epsilon$, we get
\begin{equation}
\| (A_k - A_k^r)(w_{k} - \opt) \| = \| \Big( \sum_{i\in V} \nabla f_i(w_{k}) - \sum_{j\in S} \gamma_j \nabla f_j(w_{k})\Big) + \sum_{j \in S} \gamma_j f_j(\opt) \| \leq 2\epsilon,
\end{equation}
where $\epsilon$ is the estimation error of the full gradient by the weighted gradients of the elements of the subset $S$, and we used $\|\sum_{i\in V} \nabla f_i(w_{k}) - \sum_{j\in S} \gamma_j \nabla f_j(w_{k}) \|\leq \epsilon$.

Substituting Eq. \ref{eq:ak} into Eq. \ref{eq:cmp} we obtain
\begin{equation}
w_{k+1} - \opt = (I-\alpha_k A_k^r)(w_{k}-\opt) + \alpha_k e_k
\end{equation}
Taking norms of both sides, we get
\begin{equation}\label{eq:norm}
\| w_{k+1}-\opt \| \leq \| (I - \alpha_k A_k^r) (w_k-\opt \|) + \alpha_k \| e_k \|
\end{equation}

Now, we have
\begin{align}
\hspace{-.5cm}
\| (I - \alpha_k A_k^r)(w_k - \opt) \| 
&= \| I(w_k - \opt) - \alpha_k A_k^r (w_k - \opt) \|\\
&= \| I(w_k - \opt) - \alpha_k (A_k^r - A_k) (w_k - \opt) - \alpha_k A_k (w_k - \opt) \|\\
&\leq\| (I-\alpha_k A_k)(w_k - \opt)\| + \alpha_k \|(A_k - A_k^r) (w_k - \opt)\| \\
&\leq \| (I-\alpha_k A_k)(w_k - \opt)\| + 2\alpha_k \epsilon
\end{align}

Substituting into Eq. \ref{eq:norm}, we obtain
\begin{equation}
\| w_{k+1}-\opt \| \leq \| I - \alpha_k A_k \| \cdot \| w_k-\opt \| + 2\alpha_k \epsilon + \alpha_k \| e_k \|
\end{equation}

From strong convexity of $\sum_{i \in V} f_i(w)$, and gradient smoothness of each component $f_i(w)$ we have
\begin{equation}
\mu I_n \preceq \sum_{i\in V} \nabla^2 f_i(w), A_k \preceq \beta I_n, \quad x \in \mathcal{X},
\end{equation}
where $\beta=\sum_{i \in V} \beta_i$ In addition, from the gradient smoothness of the components we can write
\begin{align}
\| e_k \| 
&\leq \sum_{j\in S} \gamma_j \beta_j \| w_{k} - w_{j}^k \| \\
&\leq \sum_{j\in S} \gamma_j \beta_j \sum_{i=1}^{j-1} \| w_{i-1}^k - w_{i}^k \| \\
&\leq \sum_{j\in S} \gamma_j \beta_j \alpha_k \sum_{i=1}^{j-1} \| \gamma_i \nabla f_i(w_{i}^k) \| \\
&\leq \alpha_k \beta C r \gamma_{\max}^2,
\end{align}
where in the last line we used $|S|=r$. Therefore,
\begin{align}\label{eq:smooth}
\| w_{k+1} - \opt \| 
&\leq \max(\| 1-\alpha_k \mu\| , \| 1-\alpha_k \beta \|) \| w_{k} - \opt \| + 2\alpha_k \epsilon + \alpha_k^2 \beta C r \gamma_{\max}^2\\
&\leq (1-\alpha_k \mu) \| w_{k} - \opt \| +  2\alpha_k \epsilon + \alpha_k^2 \beta C r \gamma_{\max}^2 \quad \text{if} \quad \alpha_k \beta \leq 1.
\end{align}

For $0<\tau\leq1$, the theorem follows by applying Lemma \ref{lemma:seq} to Eq. \ref{eq:smooth} with $c=\alpha \mu$, $e=2\alpha\epsilon$, $d=\alpha^2\beta C r \gamma_{\max}^2$.
For $\tau=0$, where we have a constant step size $\alpha_k = \alpha \leq  \frac{1}{\beta}$, we get
\begin{align}
\hspace{-5mm}\| w_{k+1} - \opt \| 
&\leq (1-\alpha \mu)^{k+1} \| w_{k} - \opt \| +  2\alpha\epsilon \sum_{i=0}^{k} (1-\alpha \mu)^{i} + \alpha^2 \sum_{i=0}^{k} (1-\alpha \mu)^{i} \beta C r \gamma_{\max}^2\\
&\leq (1-\alpha \mu)^{k+1} \| w_{k} - \opt \| +  2\epsilon/\mu + \alpha \beta C r \gamma_{\max}^2/\mu, \label{eq:tmp}\\
&\leq (1-\alpha \mu)^{k+1} \| w_{k} - \opt \| +  2\epsilon/\mu + C r \gamma_{\max}^2/\mu,
\end{align}
where the inequality in Eq. \ref{eq:tmp} follows since $\sum_{i=0}^{k} (1-\alpha \mu)^{i} \leq \frac{1}{\alpha \mu}$.

\section{Norm of the Difference Between Gradients}

\subsection{Convex Loss Functions}\label{app:grad_bound}
For ridge regression $f_i(w)=\frac{1}{2} ( \langle x_i, w \rangle - y_i )^2 + \frac{\lambda}{2} \| w\|^2$, we have $\nabla f_i(w) = x_i (\langle x_i, w \rangle - y_i) + \lambda w$.
Therefore,
 \begin{eqnarray}
 \| \nabla f_i(w)-\nabla f_j(w)\| = (\| x_i - x_j \|. \| w \|+ \| y_i - y_j \|) \| x_j\|
 \end{eqnarray}
For $\| x_i \| \leq 1$, and 
$|y_i-y_j|\approx0$
we get 
  \begin{eqnarray}
 \| \nabla f_i(w)-\nabla f_j(w)\| \leq \| x_i - x_j \| O(\| w \|)
 \end{eqnarray}

For reguralized logistic regression with $y \in \{-1,1\}$, we have $\nabla f_i(w) = y_i/(1+e^{y_i \langle x_i, w \rangle})$. For $y_i = y_j$ we get
\begin{eqnarray}
\| \nabla f_i(w)-\nabla f_j(w)\| &=&
 \frac{e^{\| x_i - x_j \|.\| w \|}-1}{1+e^{-\langle x_i, x \rangle}} \| x_j \|.
\end{eqnarray}
For $\| x_i \| \leq 1$, using Taylor approximation $e^x \leq 1+x$, and noting that $\frac{1}{1+e^{-\langle x_i, w \rangle}} \leq1$ we get
 \begin{eqnarray}
\| \nabla f_i(w)-\nabla f_j(w)\| \leq \frac{\| x_i - x_j \|.\| w \|}{1+e^{-\langle x_i, w \rangle}}\| x_j \| \leq \| x_i - x_j \| O(\| w \|).
\end{eqnarray}

For classification, we require $y_i=y_j$, hence we can select subsets from each class and then merge the results. On the other hand, in ridge regression we also need $|y_i - y_j|$ to be small.
Similar results can be deduced for other loss functions including square loss, smoothed hinge loss, etc.

Assuming $\| w \|$ is bounded for all $w \in \mathcal{W}$, upper-bounds on the euclidean distances between the gradients can be pre-computed.

\subsection{Neural Networks}\label{app:grad_bound_nn}
Formally, consider an $L$-layer perceptron, where $w^{(l)}\in \mathbb{R}^{M_l \times M_{l-1}}$ is the weight matrix for layer $l$ with $M_l$ hidden units, and $\sigma^{(l)}(.)$ be a Lipschitz continuous activation function. Then, let
\begin{eqnarray}
    x_i^{(0)}&=&x_i, \\
    z_i^{(l)}&=&w^{(l)}x_i^{(l-1)}, \\
    x_i^{(l)}&=&\sigma^{(l)} (z_i^{(l)}).
\end{eqnarray}
With 
\begin{eqnarray}
\Sigma'_l(z_i^{(l)})&=&\text{diag}\big( \sigma'^{(l)}(z_{i,1}^{(l)}), \cdots \sigma'^{(l)}(z_{i,M_l}^{(l)}) \big),\\ \Delta_i^{(l)}&=&\Sigma'_l(z_i^{(l)}) w^T_{l+1} \cdots \Sigma'_l(z_i^{(L-1)})w^T_{L},
\end{eqnarray} 
we have
\begin{align}
    \| \nabla f_i(w) - \nabla & f_j(w) \| \nonumber \\
    = &\| \big( \Delta_i^{(l)}\Sigma'_L(z_i^{(L)}) \nabla f_i^{(L)}(w) \big) \big( x_i^{(l-1)} \big)^T 
    -\big( \Delta_j^{(l)}\Sigma'_L(z_j^{(L)}) \nabla f_j^{(L)}(w) \big) \big( x_j^{(l-1)} \big)^T\| \\
    \leq&\| \Delta_i^{(l)} \| \cdot \| x_i^{(l-1)} \| \cdot \|\Sigma'_L(z_i^{(L)})  \nabla f_i^{(L)}(w)- \Sigma'_L(z_j^{(L)}) \nabla f_j^{(L)}(w) \| \nonumber\\
    & + \| \Sigma'_L(z_j^{(L)})  \nabla f_i^{(L)}(w) \|\cdot \| \Delta_i^{(l)} \big( x_i^{(l-1)} \big)^T - \Delta_j^{(l)} \big( x_j^{(l-1)} \big)^T \| \\
    \leq&\| \Delta_i^{(l)} \| \cdot \| x_i^{(l-1)} \| \cdot \|\Sigma'_L(z_i^{(L)})  \nabla f_i^{(L)}(w)- \Sigma'_L(z_j^{(L)}) \nabla f_j^{(L)}(w)\| \nonumber\\
    & + \| \Sigma'_L(z_j^{(L)})  \nabla f_i^{(L)}(w) \|\cdot \big(\| \Delta_i^{(l)}\| \cdot \| x_i^{(l-1)} \| + \|\Delta_j^{(l)}\| \cdot \| x_j^{(l-1)} \| \big) \\
    \leq& \underbrace{\max_{l,i}\big(\| \Delta_i^{(l)} \| \cdot \| x_i^{(l-1)} \|\big)}_{c_1} \cdot \|\Sigma'_L(z_i^{(L)})  \nabla f_i^{(L)}(w)- \Sigma'_L(z_j^{(L)}) \nabla f_j^{(L)}(w)\| \nonumber\\
    & + \underbrace{\| \Sigma'_L(z_i^{(L)})  \nabla f_i^{(L)}(w) \|\cdot \max_{l,i,j}\big(\| \Delta_i^{(l)}\| \cdot \| x_i^{(l-1)} \| + \|\Delta_j^{(l)}\| \cdot \| x_j^{(l-1)} \| \big)}_{c_2}
\end{align}
Various weight initialization \cite{glorot2010understanding} and activation normalization techniques \cite{ioffe2015batch, ba2016layer} uniformise the activations across samples. As a result, the variation of the gradient norm is mostly captured by the gradient of the loss function with respect to the pre-activation outputs of the last layer of our neural
network \cite{katharopoulos2018not}.
Assuming $\| \Sigma'_L(z_i^{(L)})  \nabla f_i^{(L)}(w) \|$ is bounded, we get 
\begin{align}
\| \nabla f_i(w) - \nabla & f_j(w) \| \leq c_1\|\Sigma'_L(z_i^{(L)})  \nabla f_i^{(L)}(w)- \Sigma'_L(z_j^{(L)}) \nabla f_j^{(L)}(w)\|+c_2,
\end{align}
where $c_1, c_2$ are constants.
The above analysis holds for any affine operation followed by a slope-bounded non-linearity $(|\sigma'(w)| \leq K )$.

\hide{
\section{Additional Experiments}
\begin{figure}
    \centering
    \includegraphics[width=.6\textwidth]{iclr2020-fast/Fig/cov_rebut.png}
    \caption{Loss residual vs. time for IG on 10 subsets of size 10\%, 20\%, 30\%, ..., 100\% selected by \alg, random subsets, and random subsets weighted by $|V|/|S|$. Stepsizes are tuned for every subset separately, by preferring smaller training loss from a large number of parameter combinations for two types of learning scheduling: exponential decay $\eta(t) = \eta_0 a^{\left \lfloor{t/n}\right \rfloor}$ with parameters $\eta_0$ and $a$ to adjust and $t$-inverse $\eta(t) = \eta_0(1 + b {\left \lfloor{t/n}\right \rfloor})^{-1}$ with $\eta_0$ and $b$ to adjust.}
    \label{fig:my_label}
\end{figure}

\begin{figure}
    \centering
    \includegraphics[width=1\textwidth]{iclr2020-fast/Fig/mnist_rebut.png}
    \caption{Training loss and test accuracy for SGD applied to full MNIST vs. subsets of size 60\% selected by \alg and random subsets of size 60\%. Both the random subsets and the subsets found by \alg change at the beginning of every epoch.}
    \label{fig:my_label}
\end{figure}
}

\hide{
\section{Additional Experiments: Training ResNet-32 with Only 1\% of Data per Epoch}
Figure~\ref{fig:cifar_0.01} shows test accuracy vs. number of epochs for training ResNet-32 on CIFAR10.
At the beginning of every epoch a subset of size 1\% is chosen at random or by \alg from the training data. The network is trained using SGD only on the selected subset for that epoch.
We used the standard learning rate schedule for training ResNet-32 on CIFAR10, i.e., we start with initial learning rate of 0.1, and exponentially decay the learning rate by a factor of 0.1 at epochs 100 and 150. 
We can see that SGD on subsets selected by \alg achieved a considerably better generalization performance compared to that of SGD on random subsets.
\begin{figure}[t]
    \centering
    \includegraphics[width=0.5\textwidth]{Fig/resnet32_0.01.pdf}
    \caption{Test accuracy vs. number of epochs for training ResNet-32 on CIFAR10. 
  At the beginning of ever epoch, a new subset of size 1\% of the data is selected by \alg and at random from the training set. SGD is then applied to the selected subsets.}
    \label{fig:cifar_0.01}
\end{figure}
}

\end{document}